\documentclass{article}

\usepackage[round, authoryear]{natbib}
    \PassOptionsToPackage{numbers, compress}{natbib}

\usepackage[final]{neurips_2022}

\usepackage[utf8]{inputenc} 
\usepackage[T1]{fontenc}    
\usepackage{hyperref}       
\usepackage{url}            
\usepackage{booktabs}       
\usepackage{amsfonts}       
\usepackage{nicefrac}       
\usepackage{microtype}      
\usepackage{xcolor}         

\usepackage{amsmath}
\usepackage{amssymb}
\usepackage{amsthm}

\usepackage{enumerate}
\usepackage{color}

\usepackage{graphicx}
\usepackage{bm}
\usepackage[toc,page]{appendix}

\usepackage{subcaption}
\usepackage{wrapfig}
\usepackage{adjustbox}
\usepackage{relsize}

\graphicspath{{plots/}}

\usepackage{algorithm}
\usepackage[noend]{algpseudocode}

\algnewcommand\algorithmicinput{\textbf{Initialize:}}
\algnewcommand\INPUT{\item[\algorithmicinput]}

\usepackage{thm-restate}

\newtheorem{lemma}{Lemma}

\newenvironment{proofsketch}{%
  \proof}{\endproof}

\usepackage{amsmath,amsfonts,bm}

\DeclareMathOperator*{\E}{\mathbb{E}}

\newcommand{\cS}{\mathcal{S}}

\newcommand{\cA}{\mathcal{A}}

\title{Aggressive Q-Learning with Ensembles:\\Achieving Both High Sample Efficiency and High Asymptotic Performance}

\author{
Yanqiu Wu$^1$ \And Xinyue Chen$^{2, 3}$ \AND
Che Wang$^1$ \And Yiming Zhang$^1$ 
\And Keith Ross$^{1, 2}$\thanks{Correspondence to: Keith Ross <keithwross@nyu.edu>.}\\
\AND
$^1$ \text{\normalfont 
New York University}\\ 
$^2$
New York University Shanghai\\
$^3$ University of California, Berkeley
}

\begin{document}

\maketitle

\begin{abstract}
Recent advances in model-free deep reinforcement learning (DRL) show that simple model-free methods can be highly effective in challenging high-dimensional continuous control tasks. In particular, Truncated Quantile Critics (TQC) achieves state-of-the-art asymptotic training performance on the MuJoCo benchmark with a distributional representation of critics; and Randomized Ensemble Double Q-Learning (REDQ) achieves high sample efficiency that is competitive with state-of-the-art model-based methods using a high update-to-data ratio and target randomization. In this paper, we propose a novel model-free algorithm, Aggressive Q-Learning with Ensembles (AQE), which improves the sample-efficiency performance of REDQ and the asymptotic performance of TQC, thereby providing overall state-of-the-art performance during all stages of training. Moreover, AQE is very simple, requiring neither distributional representation of critics nor target randomization. The effectiveness of AQE is further supported by our extensive experiments, ablations, and theoretical results.
\end{abstract}

\section{Introduction}
Off-policy Deep Reinforcement Learning algorithms aim to improve sample efficiency by reusing past experience. A number of off-policy Deep RL algorithms have been proposed for control tasks with continuous state and action spaces, including Deep Deterministic Policy Gradient (DDPG), Twin Delayed DDPG (TD3) and Soft Actor Critic (SAC) \citep{ddpg-LillicrapHPHETS15,td3-fujimoto18a,sac-v80-haarnoja18b, sac-adapt-haarnoja2018sacapps}.
TD3 introduced clipped double-Q learning,
and was shown to be significantly more sample efficient than popular on-policy methods for a wide range of MuJoCo benchmarks. Soft Actor Critic (SAC) has similar off-policy structures with clipped double-Q learning,  but it also employs maximum entropy reinforcement learning. SAC was shown to provide excellent sample efficiency and asymptotic performance in a wide-range of MuJoCo environments, including the high-dimensional Humanoid environment for which both DDPG and TD3 perform poorly. 

More recently, \citet{tqc-v119-kuznetsov20a} proposed Truncated Quantile Critics (TQC), a model-free algorithm which includes distributional representations of critics, truncation of critics prediction, and ensembling of multiple critics.  Instead of the usual modeling of the Q-function of the expectation of return, TQC approximates the distribution of the return random variable conditioned on the state and action. By dropping several of the top-most ``atoms'' and varying the number of dropped atoms of the return distribution approximation, TQC can control the over-estimation bias. TQC's asymptotic performance (that is after a long period of training) was shown to be better than that of SAC on the continuous control MuJoCo benchmark suite, including a 25\% improvement on the most challenging Humanoid environment. However, TQC is not sample efficient in that it generally requires a large number of samples to reach even moderate performance levels. 
 
\citet{redq-chen2021randomized} proposed Randomized Ensembled Double Q-learning(REDQ),  a model-free algorithm which includes a high Update-To-Data (UTD) ratio, an ensemble of Q functions, and in-target minimization across a random subset of Q functions from the ensemble. Using a UTD ratio much larger than one, meaning that several gradient steps are taken for each environment interaction, improves sample efficiency, while the ensemble and in-target minimization allows the algorithm to maintain stable and near-uniform bias under the high UTD ratio. 
The algorithm was shown to attain much better performance than SAC  at the early stage of training, and to match or improve 
the sample-efficiency of the state-of-the-art model-based algorithms for the MuJoCo benchmarks. However, although REDQ is highly sample efficient for early-stage training, its asymptotic performance is significantly below that of TQC.

Is it possible to design a simple, streamlined model-free algorithm which can achieve REDQ's high sample efficiency in early-stage training and also achieve TQC's high asymptotic performance in late stage training? In this paper, we achieve this goal with a new model-free algorithm, Aggressive Q-Learning with Ensembles (AQE). Like TQC and REDQ, AQE uses an ensemble of Q-functions, and like REDQ it uses a UTD ratio $>1$. However AQE is very simple, requiring neither distributional representation of critics as in TQC nor target randomization and double-Q learning as in REDQ. 
AQE controls overestimation bias and the standard deviation of the bias by varying the number of ensemble members $N$ and the number of ensembles $K \leq N$ that are kept when calculating the targets. 

Through carefully designed experiments, we provide a detailed analysis of AQE. We perform extensive and comprehensive experiments for both MuJoCo and DeepMind Control Suite (DMC) environments. We first show that for the five most challenging MuJoCo benchmark, AQE provides state-of-the-art performance, surpassing the performance of SAC, REDQ, and TQC at {\em all stages of training.} When averaged across the five MuJoCo environments, AQE's early stage performance is 2.9 times better than SAC, 1.6 times better than TQC and 1.1 times better than REDQ.
AQE's asymptotic performance is 26\%, 22\%, and 6\% higher than SAC, REDQ, and TQC, respectively. Then we provide additional experimental results for the nine most challenging DeepMind Control Suite (DMC) environments, which TQC and REDQ did not consider. We show that AQE also provides state-of-the-art performance at both early stage and late-stage of training. When averaged over nine environments, AQE's early stage performance is 13.71 times better than SAC, 7.59 times better than TQC and 1.02 times better than REDQ. 
AQE's asymptotic performance is 37\% better than SAC, 3\% better than REDQ, and 8\% better than TQC. We also perform an ablation study, and show that AQE is robust to choices of hyperparameters: AQE can work well with small ensembles consisting of 10-20 ensemble members, and performance does not vary significantly with small changes in the keep parameter $K$. We show that that AQE performs better than several variations, including using the median of all ensemble members and removing the most extreme minimum and maximum outlier in the targets. In order to improve computational time, we also consider different multi-head architectures for the ensemble of critics: consistent with the supervised convolutional network literature, we find that a two-head architecture not only reduces computational time but can actually improve performance for some environments. Additionally, we show that AQE continues to out-perform SAC and TQC even when these algorithms are made aggressive with a UTD $\gg 1$. 


To ensure a fair comparison and to obtain reliable and reproducible results \citep{Henderson2018DeepRL, reproduce-abs-1708-04133, duan16-pmlr-v48}, we provide open source code\footnote{https://github.com/AutumnWu/Aggressive-Q-Learning-with-Ensembles}. For all algorithmic comparisons, we use the the authors' code.

\section{Additional Related Work}


Overestimation bias due to in target maximization in Q-learning can significantly slow down learning \citep{thrun1993issues}. For tabular Q-learning, \cite{dqn-Hasselt10} introduced Double Q-Learning, and showed that it removes the overestimation basis and in general leads to an under-estimation bias. \cite{dqn-drl} showed that adding Double Q-learning to deep-Q networks 
can have a similar effect, leading to a major performance boost
for the Atari games benchmark. As stated in the Introduction, for continuous-action spaces, TD3 and SAC address the overestimation bias using clipped-double Q-learning, which brings significant performance improvements \citep{td3-fujimoto18a,sac-v80-haarnoja18b, sac-adapt-haarnoja2018sacapps}.

As mentioned in the Introduction, \citet{tqc-v119-kuznetsov20a} 
control the over-estimation bias by estimating the distribution of the return random variable, and then by dropping several of the top-most ``atoms'' from the estimated distribution. 
The distribution estimate is based on a methodology developed in \cite{ pmlr-v70-bellemare17a, DBLP:conf/icml/DabneyOSM18, QR-DQN-aaai/DabneyRBM18}, which employs an asymmetric Huber loss function to minimize the Wasserstein distance between the neural network output distribution and the target distribution. 
In this paper, in order to counter over-estimation bias, we also drop the top-most estimates, although we do so solely with an ensemble of Q-function mean estimators rather than with an ensemble of the more complex distributional models employed in \citep{tqc-v119-kuznetsov20a}.  

It is well-known that using ensembles can improve the performance of DRL algorithms \citep{Fau-s11063-013-9334-5, NIPS2016_8d8818c8, sunrise-icml/LeeLSA21}. For Q-learning based methods, \citet{averagedqn-pmlr-v70-anschel17a} use the average of multiple Q estimates to reduce variance. \citet{maxmin-LanPFW20} introduced Maxmin Q-learning, which uses the minimum of all the Q-estimates rather than the average. 
\citet{rem-agarwal2020optimistic} use Random Ensemble Mixture (REM), which employs a random convex combination of multiple Q estimates.  

Model-based methods often attain high sample efficiency by using a high UTD ratio. In particular, Model-Based Policy Optimization (MBPO) \citep{mbpo-10.5555/3454287.3455409} uses a large UTD ratio of 20-40. Compared to Soft-Actor-Critic (SAC), which is model-free and uses a UTD of 1, MBPO achieves much higher sample efficiency in the OpenAI MuJoCo benchmark \citep{mujoco-6386109, brockman2016openai}. REDQ \citep{redq-chen2021randomized}, a model-free algorithm, also successfully employs a high UTD ratio to achieve high sample efficiency.

\section{Algorithm}
     
We propose Aggressive Q-learning with Ensembles (AQE), a simple  model-free algorithm which provides state-of-the-art performance for the MuJoCo benchmark for both early and late stage of training. The pseudocode can be found in Algorithm \ref{alg:aqe}. As is the case with most off-policy continuous-control algorithms, AQE has a single actor (policy network) and multiple critics (Q-function networks), and employs Polyak averaging of the target parameters to enhance stability. Building on this algorithmic base, it also employs an update-to-data ratio $G > 1$, an ensemble of $N \geq 3$ Q-functions (rather than just two as in TD3 and SAC), and targets that average all the Q-functions excluding the Q-functions with the highest $N-K$ values. In the Appendix, we demonstrate theoretically in the tabular case of the algorithm that we can control over-estimation through adjusting $K$ and $N$. More concretely, we can bring the bias term from above zero (i.e. overestimation) to under zero (i.e. underestimation) by decreasing $K$ and/or increasing $N$. 

For comparison, REDQ employs two randomly chosen ensemble members when calculating the target, the bias does not depend on the number of ensemble models $N$ \citep{redq-chen2021randomized}. 
As discussed in Appendix, with $M = 2$ fixed, increasing the size of ensemble $N$ with the multi-head architecture does not necessarily improve the performance of REDQ. Unlike REDQ, AQE can control the bias term through both the number of ensemble models used in the average calculation $K$ and the total number of ensembles $N$, allowing for more flexibility. One other drawback for REDQ is that it ignores the estimates of all other ensemble estimates except for the minimal one in the randomly chosen set, which diminishes the power of the multiple ensemble sets. In contrast, AQE utilizes most of the ensemble models when calculating the target.The resulting algorithm is not only simple and streamlined, but also provides state-of-the art performance. For exploration, it uses entropy maximization as in SAC, although it could easily incorporate alternative exploration schemes.

AQE has three key hyperparameters, $G$, $N$, and $K$. If we set $N=2$, $K=1$ and $G= 1$, AQE is simply the underlying off-policy algorithm such as SAC.  When $N > 2$, $K=1$ and $G = 1$, then AQE becomes similar to, but not equivalent to, Maxmin Q-learning \citep{maxmin-LanPFW20}.

\begin{algorithm*}[htb]
    \caption{Aggressive Q-Learning with Ensembles}
	\label{alg:aqe}
\begin{algorithmic}[1]
		\INPUT Initial policy parameters $\theta$, $\textcolor{red}{N}$ Q-function parameters $\phi_i, i = 1, $\dots$, N$, empty replay buffer $\mathcal{D}$. Set target parameters $\phi_{\text{targ},i} \gets \phi_i$ for $i = 1, 2, $\dots$, N.$
		\Repeat
	    \State Take one action $a_t \sim \pi_\theta(\cdot|s_t)$. Observe reward $r_t$, new state $s_{t+1}$. 
	    \State Add data to replay buffer: $\mathcal{D} \gets \mathcal{D} \cup \{(s_t,a_t,r_t,s_{t+1})\}$
		\For {$\textcolor{red}{G}$ updates}
		\State Randomly sample a mini-batch $B = \{ (s,a,r,s') \}$ from $\mathcal{D}$.
		\For {each $(s,a,r,s') \in B$}
		\State Sample $\tilde{a}' \sim \pi_\theta(\cdot|s')$.
		\State Determine the $\textcolor{red}{K}$ indices from $i = 1,\dots, N$ that minimize $Q_{\text{target},i}(s',\tilde{a}')$.
		\State Compute the Q target $y$: \\
		\hskip1.5em $y(s,a) = r + \gamma \bigg( \frac{1}{K} \displaystyle\sum_{i \in K } Q_{\phi_{\text{targ},i}}(s',\tilde{a}') - \alpha \log \pi_\theta(\tilde{a}'|s') \bigg)$
		\EndFor
	    \For {$i=1, \dots, N$}
        \State Update $\phi_i$ with gradient descent using \\
        
		\hskip9.5em $\nabla_{\phi_i} \frac{1}{|B|}\displaystyle\sum_{(s,a,r,s') \in B} \left( Q_{\phi_i}(s,a) - y(s,a) \right)^2$

		\State Update target networks with $\phi_{\text{targ},i} \gets \rho \phi_{\text{targ},i} + (1-\rho)\phi_i$ 
	    \EndFor
	   \EndFor
		\State Update policy parameters $\theta$ with gradient ascent using \\
		
		\hskip2.5em $\nabla_{\theta} \frac{1}{|B|}\displaystyle\sum_{s \in B} \bigg( \frac{1}{N} \sum_{i=1}^N Q_{\phi_i}(s, \tilde{a}_\theta(s)) - \alpha \log \pi_\theta(\tilde{a}_\theta(s)|s) \bigg)$ \;\;\;\;\; $\tilde{a}_\theta(s) \sim \pi_\theta(\cdot|s)$
		
		\Until {Convergence}
\end{algorithmic}
\end{algorithm*}

AQE uses an ensemble of Q networks  (as does REDQ and TQC). Employing multiple networks, one for each Q-function output, can be expensive in terms of computation and memory. In order to reduce the computation and memory requirements, we combine network ensemble with multi-head architectures to generate multiple Q-function outputs. We consider $N$ separate Q networks each with $h$ heads, providing a total of $h\cdot N$ estimates. The $h$ heads from one network share all of the layers except the final fully-connected layer. In practice, we found $h$ = 2 heads works well for AQE, consistent with work in ensembles of convolutional neural networks for computer vision tasks \citep{lee-LeePCCB15}. 
When properly sharing low-level weights, multi-headed networks may not only retain the performance of full ensembles but can sometimes outperform them. In the next section, we discuss our experimental results. 

\section{Experimental Results}

We perform extensive and comprehensive experiments for two sets of popular benchmarks. First we provide experimental results for AQE, TQC, REDQ and SAC for the five most challenging MuJoCo environments, namely, Hopper, Walker2d, HalfCheetah, Ant and Humanoid. Then we provide additional experimental results for the nine most challenging DeepMind Control Suite (DMC) environments, namely, Cheetah-run, Fish-swim, Hopper-hop, Humanoid-stand, Humanoid-walk, Humanoid-run, Quadruped-walk, Quadruped-run and Walker-run. 
To make a fair comparison, the TQC, REDQ and SAC results are reproduced using the authors' open source code, and use the same network sizes and hyperparameters reported in their papers. 
In particular, for the MuJoCo environments, TQC employs 5 critic networks with 25 distributional samples for a total of 125 atoms. 
TQC drops 5 atoms per critic for Hopper, 0 atoms per critic for Half Cheetah, and 2 atoms per critic for Walker, Ant, and Humanoid. For REDQ, we also use the authors' suggested values of $N$ = 10 and $M$ = 2, where $M$ is the number of ensemble members used in the target calculation. 

The REDQ paper uses $G=20$ for the update-to-data ratio, and provides results for up to 300K environment interactions. Using such a high value for $G$ is computationally infeasible in our experimental setting, since we use 3 million environment interactions for Ant and Humanoid and over 4 million environment interactions for Humanoid-run in order to investigate asymptotic performance as well early-stage sample efficiency.
In the experiments reported here, we use a value of $G$ = 5 for both REDQ and AQE. 

For AQE, we use 10 Q-networks each with 2 heads, producing $20$ Q-values for each input. The AQE networks are the same size as those in the REDQ paper. For MuJoCo benchmark, AQE keeps 10 out of 20 values for Hopper, all 20 values for half-Cheetah, and 16 out of 20 values for Walker, Ant and Humanoid. 

\begin{figure}[h!tb]
\centering
\begin{subfigure}{0.329\textwidth}
	\centering
	\includegraphics[width=0.99\linewidth]{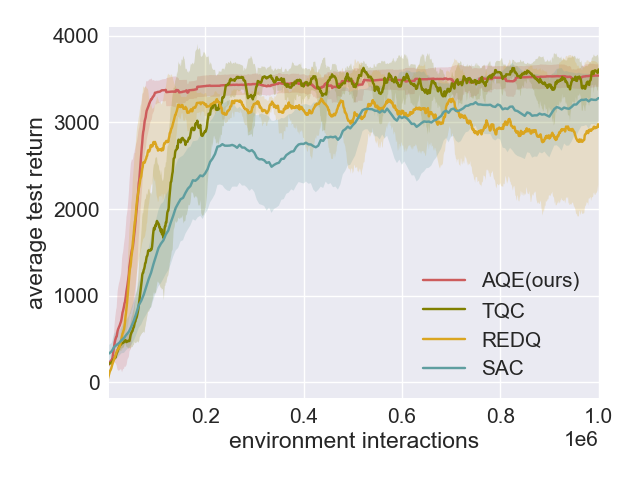}
	\caption{Hopper-v2}
	\label{fig:pri-hopper}
\end{subfigure}
\begin{subfigure}{0.329\textwidth}
	\centering
	\includegraphics[width=0.99\linewidth]{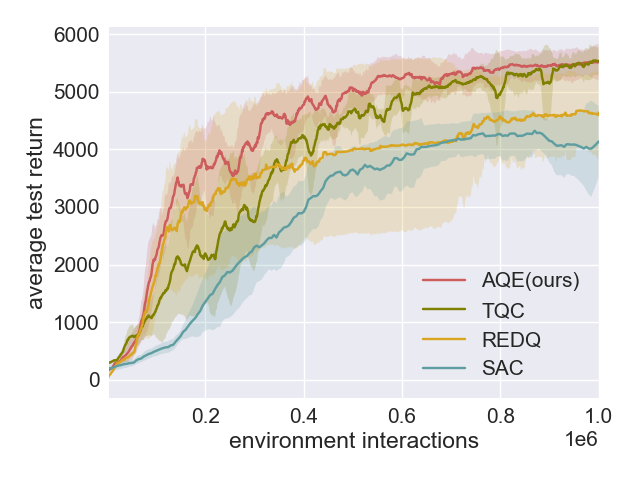}
	\caption{Walker2d-v2}
    \label{fig:pri-walker2d}
\end{subfigure}
\begin{subfigure}{0.329\textwidth}
	\centering
	\includegraphics[width=0.99\linewidth]{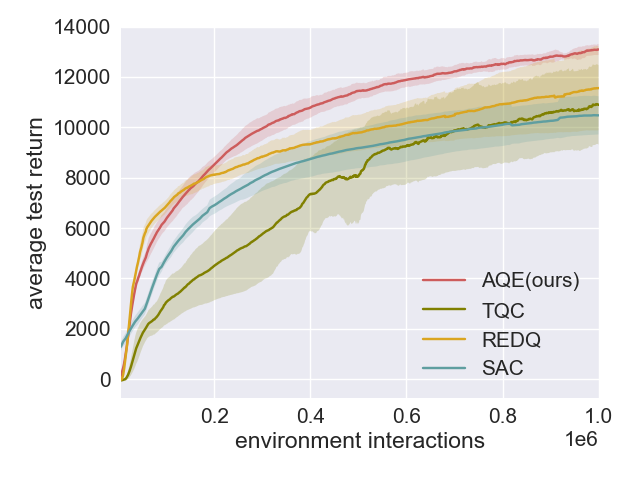}
	\caption{HalfCheetah-v2}
	\label{fig:pri-halfcheetah}
\end{subfigure}
\begin{subfigure}{0.329\textwidth}
	\centering
	\includegraphics[width=0.99\linewidth]{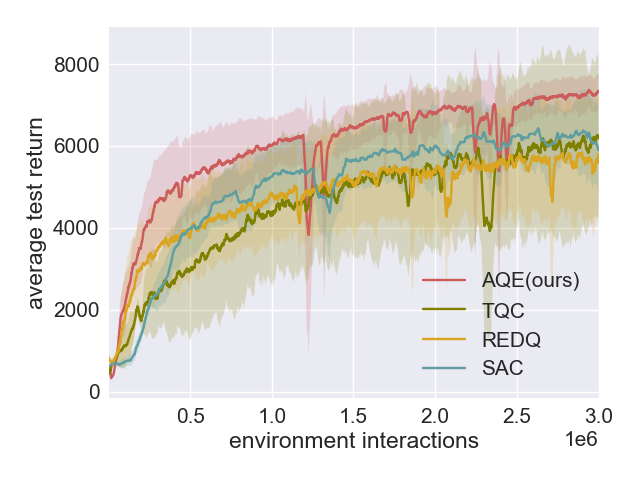}
	\caption{Ant-v2}
	\label{fig:pri-ant}
\end{subfigure}
\begin{subfigure}{0.329\textwidth}
	\centering
	\includegraphics[width=0.99\linewidth]{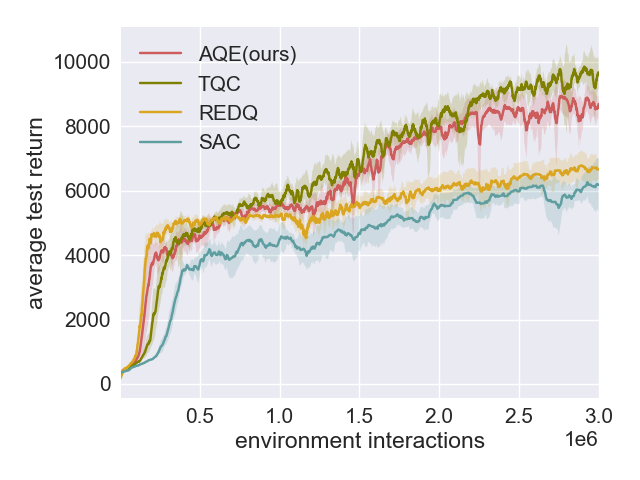}
	\caption{Humanoid-v2}
	\label{fig:pri-humanoid}
\end{subfigure}
\caption{AQE versus TQC, REDQ and SAC for MuJoCo environments. AQE is the only algorithm that beats SAC in all five environments during all stages of training, and it typically beats SAC by a wide margin.}
\label{fig:primary}
\end{figure}

Figure \ref{fig:primary} shows the training curves for AQE, TQC, REDQ, and SAC for the five MuJoCo environments. For each algorithm, we plot the average return of 5 independent trials as the solid curve, and plot the standard deviation across 5 seeds as the shaded region. For each environment, we train each algorithm for exactly the same number of environment interactions as done in the SAC paper. We use the same evaluation protocol as in the TQC paper. Specifically, after every epoch, we run ten test episodes with the current policy, record the undiscounted sum of all the rewards in each episode and take the average of the sums as the performance. A more detailed discussion on hyperparameters and implementation details is given in the Appendix. 

We see from Figure \ref{fig:primary} that AQE is the only algorithm that beats SAC in all five environments during all stages of training. Moreover, it typically beats SAC by a very wide margin. Table \ref{tab:sample-efficiency} shows that, when averaged across the five environments, AQE achieves SAC asymptotic performance  approximately 3x faster than SAC and 2x faster than REDQ and TQC.
As seen from Figure \ref{fig:primary} and Table \ref{tab:early-performance}, in the  early stages of training, AQE matches the excellent performance of REDQ in all five environments, and both algorithms are much more sample efficient than SAC and TQC.  
As seen from Figure \ref{fig:primary} and Table \ref{tab:late-performance}, in late-stage training, AQE always matches or beats all other algorithms, except for Humanoid, where TQC is about 10\% better. Table \ref{tab:late-performance} shows that, when averaged across all five environments, AQE's asymptotic performance is 26\%, 22\%, and 6\% higher than SAC, REDQ, and TQC, respectively. 

\begin{table}[htb]
    \begin{center}
    \caption{Sample efficiency comparison of SAC, TQC, REDQ and AQE. The numbers show the amount of data collected when the specified performance level is reached (roughly corresponding to 90\% of SAC's final performance). The last three columns show how many times AQE is more sample efficient than SAC, TQC and REDQ in reaching that performance level. For each task, the lowest sample use, and those within 5\% difference are highlighted.} 
    \label{tab:sample-efficiency}
    \vskip 0.1in
    \begin{tabular}{l c c c c c c c}
    Performance & SAC & TQC & REDQ & AQE & AQE/SAC & AQE/TQC & AQE/REDQ \\
    \hline
    Hopper at 3000 & 506K & 184K & 136K & \textbf{77K} & 6.57 & 2.39 & 1.77 \\
    Walker2d at 4000 & 631K & 371K & 501K & \textbf{277K} & 2.28 & 1.34 & 1.81 \\ 
    HalfCheetah at 10000 & 763K & 737K & 552K & \textbf{304K} & 2.51 & 2.42 & 1.82 \\
    Ant at 5500 & 1445K & 1759K & 1749K & \textbf{632K} & 2.29 & 2.78 & 2.77 \\
    Humanoid at 6000 & 2469K & \textbf{1043K} & 1862K & 1345K & 1.84 & 0.78 & 1.38 \\
    \hline
    Average & - & - & - & - & 3.10 & 1.94 & 1.91 \\
    \end{tabular}
    \end{center}
\end{table}

\begin{table}[ht]
    \centering
    \caption{Early-stage performance comparison of SAC, TQC, REDQ and AQE. The numbers show the performance achieved when the specific amount of data is collected. 
    On average, AQE performs 2.9 times better than SAC, 1.6 times better than TQC and 1.1 times better than REDQ. For each task, the highest score and those within 5\% difference are highlighted.}
    \vskip 0.1in
    \begin{adjustbox}{max width=\textwidth}
    \begin{tabular}{l c c c c c c c}
         Amount of data & SAC & TQC & REDQ & AQE & AQE/SAC & AQE/TQC & AQE/REDQ \\
         \hline
         Hopper at 100K & 1456 & 1807 & 2747 & \textbf{3345} & 2.30 & 1.85 & 1.22 \\
         Walker2d at 100K & 501 & 1215 & 1810 & \textbf{2150} & 4.29 & 1.77 & 1.19 \\
         HalfCheetah at 100K & 3055 & 4801 & \textbf{6876} & 6378 & 2.09 & 1.33 & 0.93 \\
         Ant at 250K & 2107 & 2344 & 3279 & \textbf{4153} & 1.97 & 1.77 & 1.27 \\
         Humanoid at 250K & 1094 & 3038 & \textbf{4535} & 3973 & 3.63 & 1.31 & 0.88 \\
         \hline
         Average at early stage & - & - & - & - & 2.86 & 1.61 & 1.10 \\
    \end{tabular}
    \end{adjustbox}
    \label{tab:early-performance}
\end{table}
\begin{table}[ht]
    \centering
    \caption{Late-stage performance comparison of SAC, TQC, REDQ and AQE. The numbers show the performance achieved when a specific amount of data is collected. The last three columns show the ratio of AQE performance compared to SAC, TQC, and REDQ performance. On average, during late-stage training, AQE performs 1.26 times better than SAC, 1.06 times better than TQC, and 1.22 times better than REDQ. For each task, the highest score and those within 5\% difference are highlighted.}
    \label{tab:late-performance}
    \vskip 0.1in
    \begin{adjustbox}{max width=\textwidth}
    \begin{tabular}{l c c c c c c c}
         Amount of data & SAC & TQC & REDQ & AQE & AQE/SAC & AQE/TQC & AQE/REDQ \\
         \hline
         Hopper at 1M & 3282 & \textbf{3612} & 2954 & \textbf{3541} & 1.08 & 0.98 & 1.20 \\
         Walker2d at 1M & 4134 & \textbf{5532} & 4637 & \textbf{5517} & 1.33 & 1.00 & 1.19 \\
         HalfCheetah at 1M & 10475 & 10887 & 11562 & \textbf{13093} & 1.25 & 1.20 & 1.13 \\
         Ant at 3M & 5903 & 6186 & 5785 & \textbf{7345} & 1.24 & 1.19 & 1.27 \\
         Humanoid at 3M & 6177 & \textbf{9593} & 6649 & 8680 & 1.41 & 0.91 & 1.31 \\
         \hline
         Average at late stage & - & - & - & - & 1.26 & 1.06 & 1.22 \\
    \end{tabular}
    \end{adjustbox}
\end{table}



\subsection{Fixed Hyperparameters across Environments}
Following the TQC paper, in Figure \ref{fig:primary} we used different drop atoms for TQC for the different MuJoCo environments.
To make the comparison fair, we also used different keep values $K$ for AQE for the different environments.
We repeat the experiment on the five MuJoCo environments, but now use the same hyperparameter values across environments for TQC (drop two atoms per network) and AQE ($K$ = 16). These choices of fixed hyperparameters appear to give the best overall performance for the two algorithms. The training curves for AQE, TQC, REDQ, and SAC and detailed early-stage and late-stage performance comparisons of all algorithms for this experiment are shown in the Appendix. 

We can see from the results that with fixed hyperparamters, the conclusions for AQE remain largely unchanged, except for Hopper, where REDQ becomes the strongest algorithm. 
When averaged across environments,  AQE still matches the high sample efficiency of REDQ during the early stages of training. Furthermore, 
on average, AQE's asymptotic performance is still 16\%, 11\% and 9\% higher than SAC, REDQ and TQC, respectively.

\subsection{DeepMind Control Suite Results}
In this section, we provide detailed experimental results for AQE, TQC, REDQ and SAC for the nine most challenging DeepMind Control Suite (DMC) environments. The TQC and REDQ papers do not consider DMC benchmark, so we employ the same hyperparameters for TQC and REDQ as for the MuJoCo environments. For TQC, we employ 5 critic networks with 25 distributional samples and drop 2 atoms per critic across environments. For REDQ, we keep using $N = 10$ and $M = 2$. To make the comparison fair, we also use the same hyperparameter values across environments for AQE. We present the learning curves in Figure \ref{fig:dmc-results}. Similar to MuJoCo benchmark, for each algorithm, we plot the average return of 5 independent trials as the solid curve, and plot the standard deviation across 5 seeds as the shaded region. We run all algorithms to 1 million environment interactions except for the most challenging environment, Humanoid-run, where we run up to 4.5 million environment interactions. We use the same evaluation protocol as for the MuJoCo environments. 

\begin{figure}[!htb]
\centering
\begin{subfigure}{0.3\textwidth}
	\centering
	\includegraphics[width=0.99\linewidth]{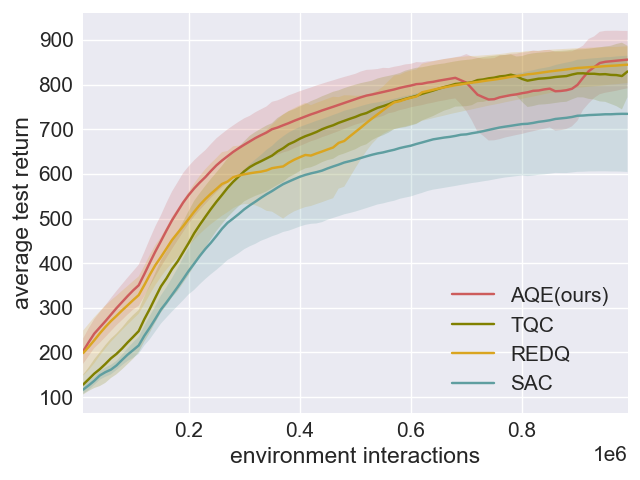} 
	\caption{Cheetah run}
	\label{fig:cheetah-run}
\end{subfigure}
\begin{subfigure}{0.3\textwidth}
	\centering
	\includegraphics[width=0.99\linewidth]{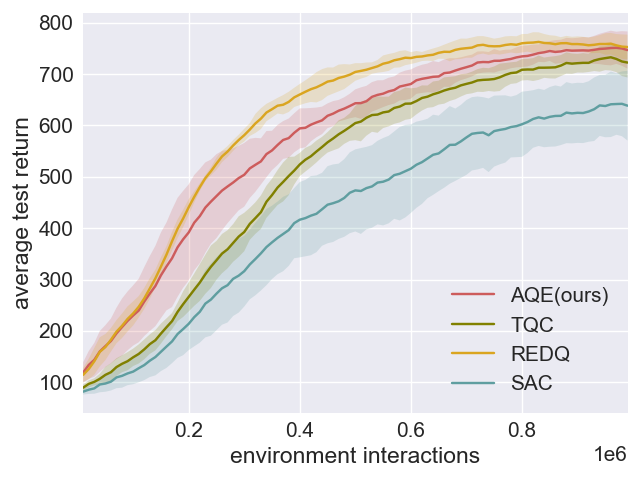}
	\caption{Fish swim}
    \label{fig:fish-swim}
\end{subfigure}
\begin{subfigure}{0.3\textwidth}
	\centering
	\includegraphics[width=0.99\linewidth]{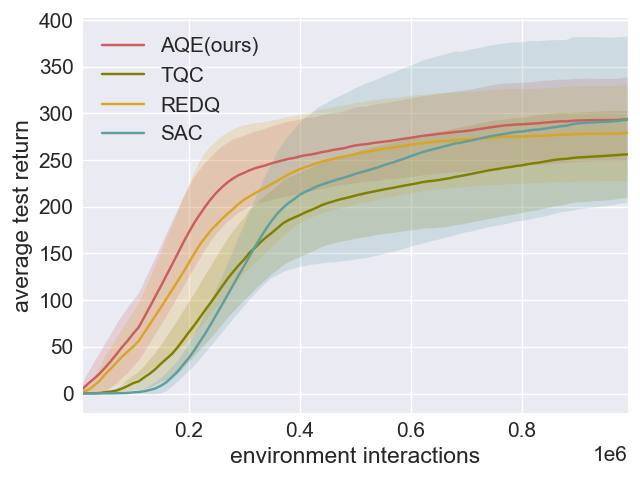}
	\caption{Hopper hop}
	\label{fig:hopper-hop}
\end{subfigure}

\begin{subfigure}{0.3\textwidth}
	\centering
	\includegraphics[width=0.99\linewidth]{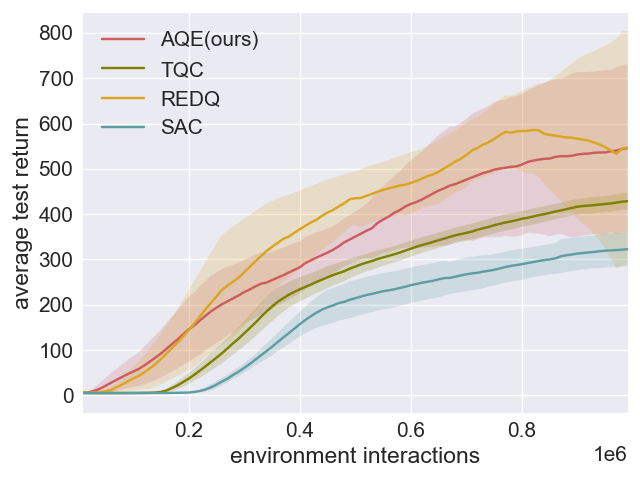}
	\caption{Humanoid stand}
	\label{fig:Humanoid-stand}
\end{subfigure}
\begin{subfigure}{0.3\textwidth}
	\centering
	\includegraphics[width=0.99\linewidth]{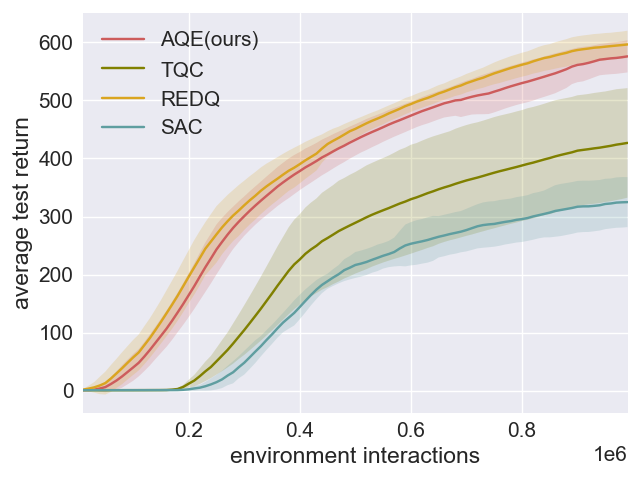}
	\caption{Humanoid walk}
    \label{fig:humanoid-walk}
\end{subfigure}
\begin{subfigure}{0.3\textwidth}
	\centering
	\includegraphics[width=0.99\linewidth]{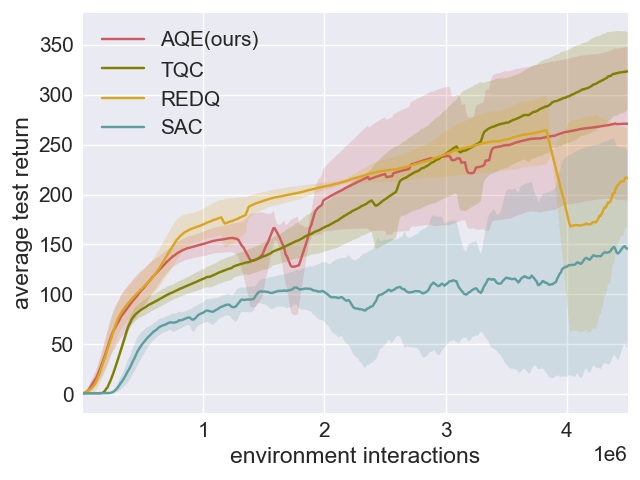}
	\caption{Humanoid run}
	\label{fig:humanoid-run}
\end{subfigure}

\begin{subfigure}{0.3\textwidth}
	\centering
	\includegraphics[width=0.99\linewidth]{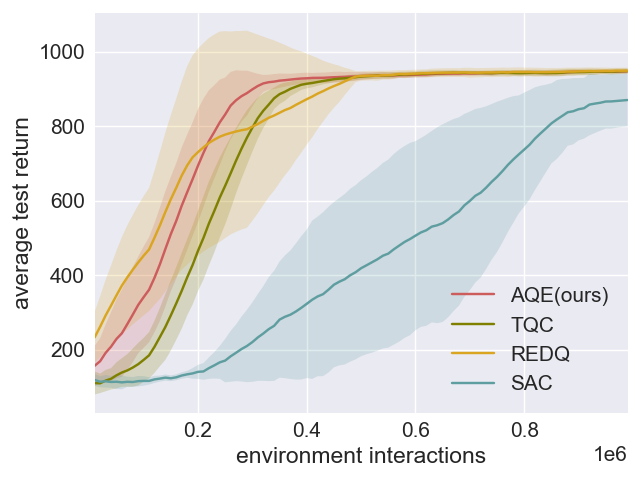}
	\caption{Quadruped walk}
	\label{fig:quadruped-walk}
\end{subfigure}
\begin{subfigure}{0.3\textwidth}
	\centering
	\includegraphics[width=0.99\linewidth]{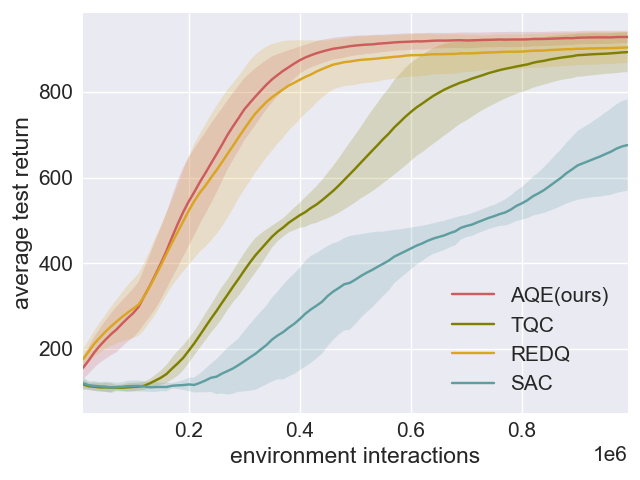}
	\caption{Quadruped run}
    \label{fig:quadruped-run}
\end{subfigure}
\begin{subfigure}{0.3\textwidth}
	\centering
	\includegraphics[width=0.99\linewidth]{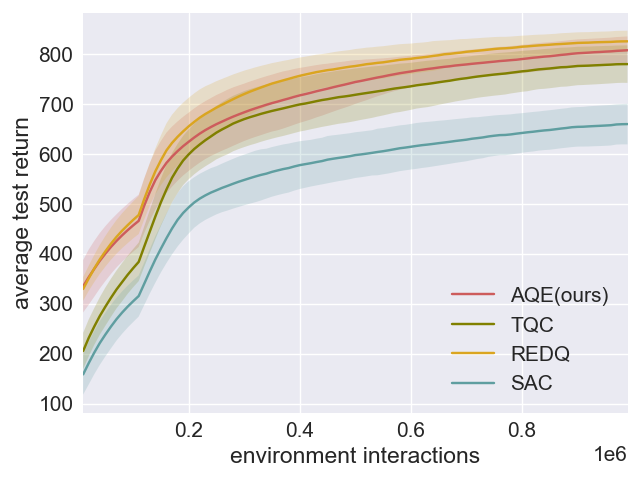}
	\caption{Walker run}
	\label{fig:walker-run}
\end{subfigure}
\caption{AQE versus TQC, REDQ and SAC in DeepMind Control Suite benchmark. AQE and TQC use same hyperparameters across the nine environments.}
\label{fig:dmc-results}
\end{figure}

Figure \ref{fig:dmc-results} shows that in DMC environments with fixed hyperparameters, AQE continues to outperform TQC except for the Humanoid-run environment, where TQC performs better than AQE in the final stage of training. AQE and REDQ have comparable results in some of the DMC environments during traning, however, AQE usually outperforms REDQ in the more challenging environments, such as Hopper-hop, Humanoid-run, and Quadruped-run. We report detailed early-stage and late-stage performance comparisons of all algorithms in Appendix. 

In summary, in the early stage of training (100K data), AQE performs 13.71x better than SAC, 7.59x better than TQC and matches the excellent performance
of REDQ in nine environments. 
In the late-stage training (1M data), AQE always matches or outperforms all other algorithms, except for Humanoid-run, where TQC performs the best. 
On average, AQE performs 37\% better than SAC, 8\% better than TQC, and 3\% better than REDQ. 
Using the same hyperparameters and averaged across nine DMC environments, AQE achieves the asymptotic performance of SAC approximately 3x faster than SAC, 1.57x faster than TQC, and 1.05x faster than REDQ.

\subsection{Ablations}
\label{sec:ablation}
In this section, we use ablations to provide further insight into AQE. We focus on the Ant environment, and run the experiments up to 1M time steps. 
(In the Appendix we provide ablations for the other four environments.)
As in the REDQ paper, we consider not only performance but normalized bias and standard deviation of normalized bias as defined by the REDQ authors. 
We first look at how the ensemble size $N$ affects AQE. The first row in Figure \ref{fig:ablation} shows AQE with $N$ equal to 2, 5, 10 and 15,  with two heads for each Q network, and the percentage of kept Q-values unchanged. As the ensemble size $N$ increases, we generally obtain a more stable average bias, a lower std of bias, and stronger performance. When trained with high UTD value, a relatively small ensemble size, for example, $N=5$, can greatly reduce bias accumulation, resulting in much stronger performance.
This experimental finding is consistent with the results in Theorem 1 in Appendix G.

\begin{figure}[h!tb]
\centering
\begin{subfigure}{0.329\textwidth}
	\centering
	\includegraphics[width=0.99\linewidth]{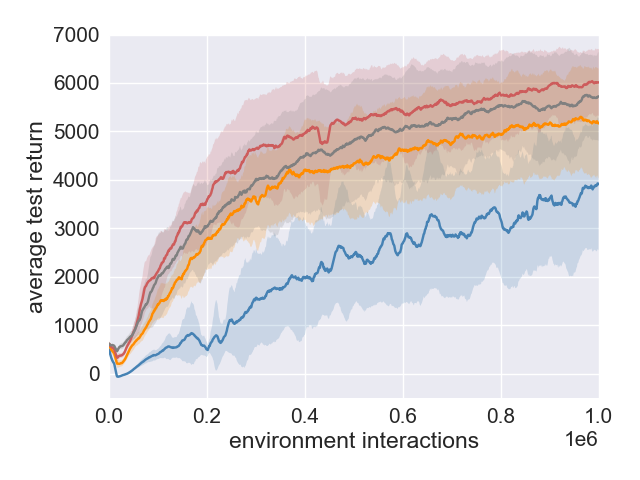}
	\caption{Ensemble size: Perf}
	\label{fig:abl-N}
\end{subfigure}
\begin{subfigure}{0.329\textwidth}
	\centering
	\includegraphics[width=0.99\linewidth]{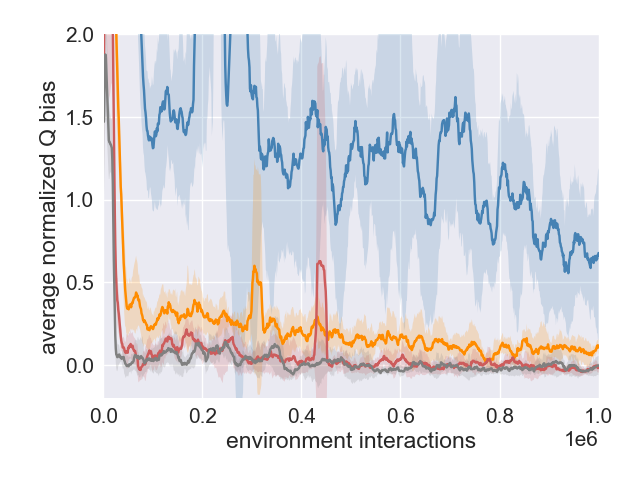}
	\caption{Ensemble size: Bias}
    \label{fig:abl-N-qbias}
\end{subfigure}
\begin{subfigure}{0.329\textwidth}
	\centering
	\includegraphics[width=0.99\linewidth]{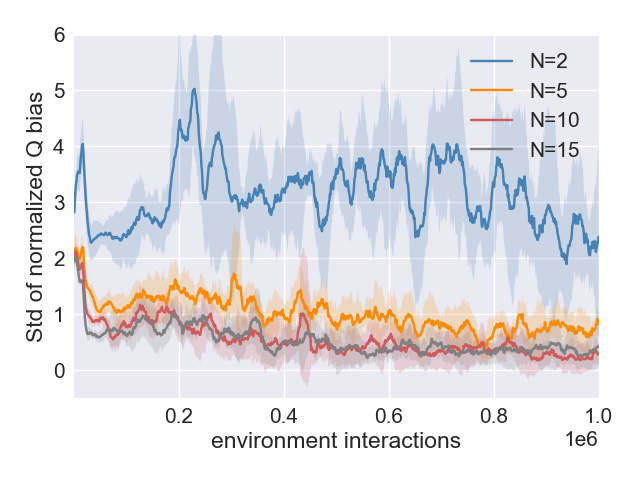}
	\caption{Ensemble size:Std}
	\label{fig:abl-N-std}
\end{subfigure}
\begin{subfigure}{0.329\textwidth}
	\centering
	\includegraphics[width=0.99\linewidth]{ablation_d.png}
	\caption{Keep value: Perf}
	\label{fig:abl-d}
\end{subfigure}
\begin{subfigure}{0.329\textwidth}
	\centering
	\includegraphics[width=0.99\linewidth]{ablation_d_qbias.png}
	\caption{Keep value: Bias}
    \label{fig:abl-d-qbias}
\end{subfigure}
\begin{subfigure}{0.329\textwidth}
	\centering
	\includegraphics[width=0.99\linewidth]{ablation_d_std.png}
	\caption{Keep value: Std}
	\label{fig:abl-d-std}
\end{subfigure}
\begin{subfigure}{0.329\textwidth}
	\centering
	\includegraphics[width=0.99\linewidth]{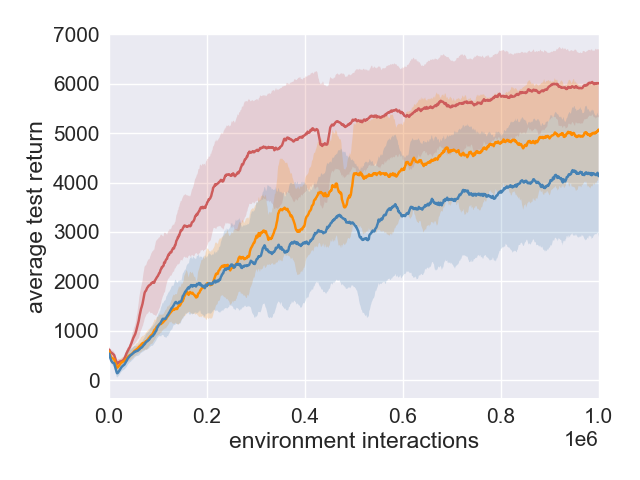}
	\caption{Variations: Perf}
	\label{fig:abl-var}
\end{subfigure}
\begin{subfigure}{0.329\textwidth}
	\centering
	\includegraphics[width=0.99\linewidth]{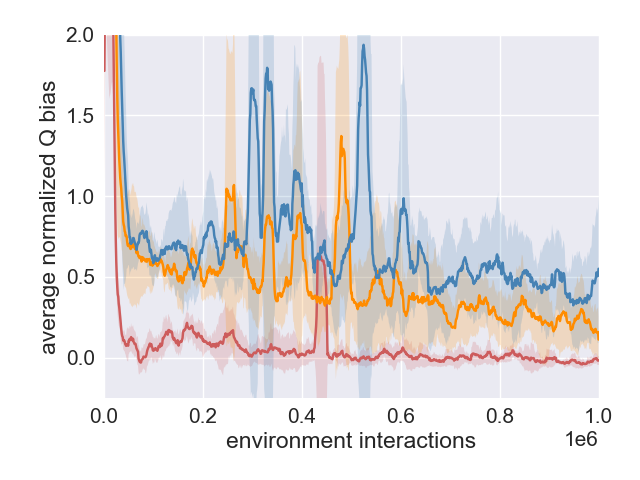}
	\caption{Variations: Bias}
    \label{fig:abl-var-qbias}
\end{subfigure}
\begin{subfigure}{0.329\textwidth}
	\centering
	\includegraphics[width=0.99\linewidth]{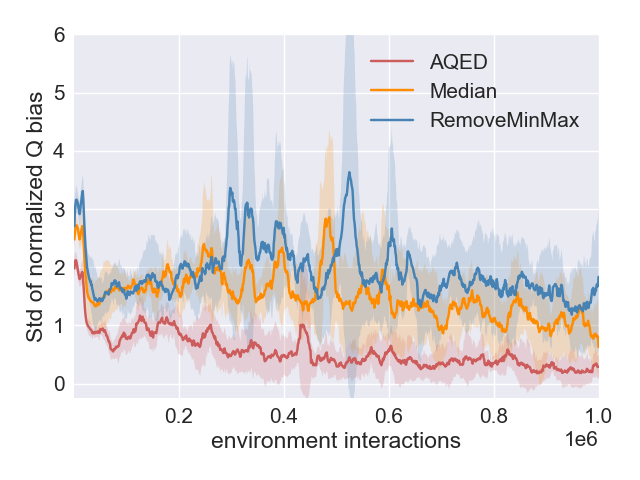}
	\caption{Variations: Std}
	\label{fig:abl-var-std}
\end{subfigure}
\caption{AQE ablation results for Ant. The top row shows the effect of the ensemble size $N$. The second row shows the effect of keep number parameter $K$. The third row compares AQE to some variants. }
\label{fig:ablation}
\end{figure}

The second row in Figure \ref{fig:ablation} shows how the keep parameter can affect the algorithm's performance: under the same high UTD value, as $K$ decreases, the average normalized Q bias goes from over-estimation to under-estimation. Consistent with the theoretical result in Theorem 1, by decreasing $K$ we lower the average bias. When $K$ becomes too small, the Q estimate becomes too conservative and starts to have negative bias, which makes learning difficult. We see that $K=16$ has an average bias closest to 0 and also a consistently small std of bias. These results are similar for the other four environments, as shown in the Appendix.

The third row in Figure \ref{fig:ablation} shows results for variants of the target computation methods. The Median curve uses the median value of all the Q estimates in the ensemble to compute the Q target. The RemoveMinMax curve drops the minimum and maximum values of all the Q estimates in the ensemble to compute the Q target. We see that these two variants give larger positive Q bias values. 

We also considered different combinations of ensemble size $N$ and the number of multi-heads $h$ while keeping the total number of Q-function estimates $N \cdot h$ fixed. We performed these experiments for all five environments.
In terms of performance, we found the two best combinations to be $N=20$, $h=1$ and $N=10$, $h=2$, with the former being about 50\% slower than the latter in terms of computation time. We also consider endowing REDQ with the same multi-head ensemble architecture as AQE and find that it does not improve REDQ substantially (additional details in Appendix F). 

We also consider comparing AQE with SAC and TQC, all using a more aggressive UTD ratio of $G=5$. Although SAC becomes more sample efficient with $G=5$, AQE continues to outperform both algorithms except for Humanoid, where once again TQC performs somewhat better than AQE for the final stage of training (Additional details in appendix E).

\section{Conclusion}

Perhaps the most important takeaway from this study is that a simple model-free algorithm can do surprisingly well, providing state-of-art performance at all stages of training. There is no need for a model, distributional representation of the return, or in-target randomization to achieve high sample efficiency and asymptotic performance. 


With extensive experiments and ablations, we show that AQE is both performant and robust. In both OpenAI Gym and DMControl, AQE is able to achieve superior performance in all stages of training with the same hyperparameters, and it can be further improved with per-task finetuning. Our ablations show that AQE is robust to small changes in the hyperparameters. Our theoretical results complement the experimental results, showing that the estimation bias can be controlled by either varying the ensemble size $N$ or the keep parameter $K$. 

AQE along with prior works show that a high update-to-data ratio combined with refined bias control can lead to very significant performance gain. An interesting future work direction is to investigate whether there are other critical factors (in addition to bias control) that can allow us to further benefit from a high update-to-data ratio, and achieve even better sample efficiency with simple model-free methods. 

\bibliographystyle{abbrvnat}
\bibliography{aqe}

\section*{Checklist}

The checklist follows the references.  Please
read the checklist guidelines carefully for information on how to answer these
questions.  For each question, change the default \answerTODO{} to \answerYes{},
\answerNo{}, or \answerNA{}.  You are strongly encouraged to include a {\bf
justification to your answer}, either by referencing the appropriate section of
your paper or providing a brief inline description.  For example:
\begin{itemize}
  \item Did you include the license to the code and datasets? \answerNo{The code and the data are proprietary.}
  \item Did you include the license to the code and datasets? \answerNA{}
\end{itemize}
Please do not modify the questions and only use the provided macros for your
answers.  Note that the Checklist section does not count towards the page
limit.  In your paper, please delete this instructions block and only keep the
Checklist section heading above along with the questions/answers below.

\begin{enumerate}

\item For all authors...
\begin{enumerate}
  \item Do the main claims made in the abstract and introduction accurately reflect the paper's contributions and scope?
    \answerYes{}
  \item Did you describe the limitations of your work?
    \answerYes{}
  \item Did you discuss any potential negative societal impacts of your work?
    \answerNA{}
  \item Have you read the ethics review guidelines and ensured that your paper conforms to them?
    \answerYes{}
\end{enumerate}

\item If you are including theoretical results...
\begin{enumerate}
  \item Did you state the full set of assumptions of all theoretical results?
    \answerYes{}
        \item Did you include complete proofs of all theoretical results?
    \answerYes{}
\end{enumerate}

\item If you ran experiments...
\begin{enumerate}
  \item Did you include the code, data, and instructions needed to reproduce the main experimental results (either in the supplemental material or as a URL)?
    \answerYes{}
  \item Did you specify all the training details (e.g., data splits, hyperparameters, how they were chosen)?
    \answerYes{}
        \item Did you report error bars (e.g., with respect to the random seed after running experiments multiple times)?
    \answerYes{}
        \item Did you include the total amount of compute and the type of resources used (e.g., type of GPUs, internal cluster, or cloud provider)?
    \answerYes{}
\end{enumerate}

\item If you are using existing assets (e.g., code, data, models) or curating/releasing new assets...
\begin{enumerate}
  \item If your work uses existing assets, did you cite the creators?
    \answerNA{}
  \item Did you mention the license of the assets?
    \answerNA{}
  \item Did you include any new assets either in the supplemental material or as a URL?
    \answerNA{}
  \item Did you discuss whether and how consent was obtained from people whose data you're using/curating?
    \answerNA{}
  \item Did you discuss whether the data you are using/curating contains personally identifiable information or offensive content?
    \answerNA{}
\end{enumerate}

\item If you used crowdsourcing or conducted research with human subjects...
\begin{enumerate}
  \item Did you include the full text of instructions given to participants and screenshots, if applicable?
    \answerNA{}
  \item Did you describe any potential participant risks, with links to Institutional Review Board (IRB) approvals, if applicable?
    \answerNA{}
  \item Did you include the estimated hourly wage paid to participants and the total amount spent on participant compensation?
    \answerNA{}
\end{enumerate}

\end{enumerate}


\newpage

\noindent\rule{\textwidth}{4pt}
\begin{center}
    \LARGE\bf Aggressive Q-Learning with Ensembles:\\Achieving Both High Sample Efficiency and High Asymptotic Performance \\ Supplementary Material
\end{center}
\noindent\rule{\textwidth}{1pt}

\appendix

\section{Hyperparameters and implementation details}
Table \ref{tab:hyper-parameter} gives a list of hyperparameters used in the experiments. Most of AQE's hyperparameters are made the same as in the REDQ paper to ensure fairness and consistency in comparisons, except that AQE has 2-head critic networks. As compared with AQE and REDQ, TQC uses a larger critic network with 3 layers of 512 units per layer. In table \ref{tab:env-dep-hp}, we report the dropped atoms $d$ for TQC and the number of Q values we keep in the ensemble to calculate the target in AQE. In algorithm \ref{alg:aqe}, we provide detailed pseudo-code. 

\begin{table}[ht]
    \centering
    \caption{Hyperparameter values.}
    \vskip 0.1in
    \begin{tabular}{l|c c c c}
    \hline
         Hyperparameters & AQE & SAC & REDQ & TQC \\
    \hline
         optimizer & \multicolumn{4}{c}{Adam\cite{}} \\
         learning rate & \multicolumn{4}{c}{$3 \cdot 10^\text{$-$4}$} \\
         discount($\gamma$) & \multicolumn{4}{c}{0.99} \\
         target smoothing coefficient($\rho$) & \multicolumn{4}{c}{0.005} \\
         replay buffer size & \multicolumn{4}{c}{$1 \cdot 10^{6}$} \\
         number of critics $N$ & 10 & 2 & 10 & 5 \\
         number of hidden layers in critic networks & 2 & 2 & 2 & 3 \\
         size of hidden layers in critic networks & 256 & 256 & 256 & 512 \\
         number of heads in critic networks $h$ & 2 & 1 & 1 & 25 \\
         number of hidden layers in policy network & \multicolumn{4}{c}{2}\\
         size of hidden layers in policy network & \multicolumn{4}{c}{256} \\
         mini-batch size & \multicolumn{4}{c}{256} \\
         nonlinearity & \multicolumn{4}{c}{ReLU} \\
         UTD ratio G & 5 & 1 & 5 & 1 \\
    \hline
    \end{tabular}
    \label{tab:hyper-parameter}
\end{table}

\begin{table}[ht]
    \centering
    \caption{Environment-dependent hyper-parameters for TQC and AQE.}
    \vskip 0.1in
    \begin{tabular}{l c c}
    \hline
        Environment & Dropped atoms per critic & Kept Q values out of $N \cdot h$ values\\
    \hline
        Hopper & 5 & 10  \\
        HalfCheetah & 0 & 20 \\ 
        Walker & 2  &  16 \\
        Ant & 2 & 16 \\
        Humanoid & 2 & 16 \\
    \hline
    \end{tabular}
    \label{tab:env-dep-hp}
\end{table}

\clearpage

\section{Additional Results for AQE, TQC, REDQ and SAC in MuJoCo Benchamrk with Fixed Hyper-parameters}

We present the experiment on the five MuJoCo environments with the same hyperparameter values across environments for TQC (drop two atoms per network) and AQE
(K = 16) in Figure \ref{fig:same-hp}. 

\begin{figure}[h!tb]
\centering
\begin{subfigure}{0.329\textwidth}
	\centering
	\includegraphics[width=0.99\linewidth]{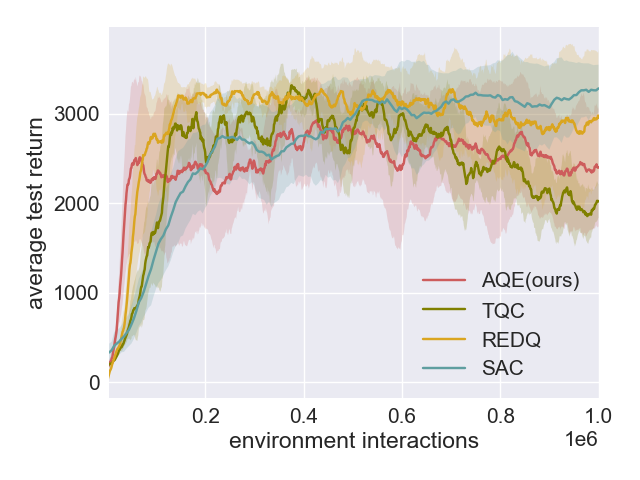}
	\caption{Hopper-v2}
	\label{fig:same-hp-hopper}
\end{subfigure}
\begin{subfigure}{0.329\textwidth}
	\centering
	\includegraphics[width=0.99\linewidth]{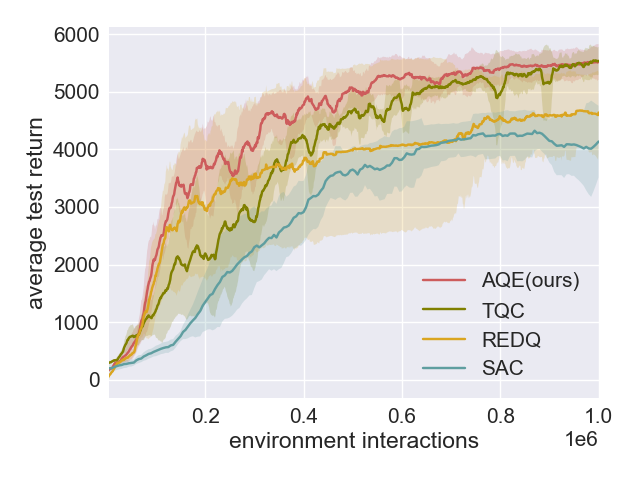}
	\caption{Walker2d-v2}
    \label{fig:same-hp-walker2d}
\end{subfigure}
\begin{subfigure}{0.329\textwidth}
	\centering
	\includegraphics[width=0.99\linewidth]{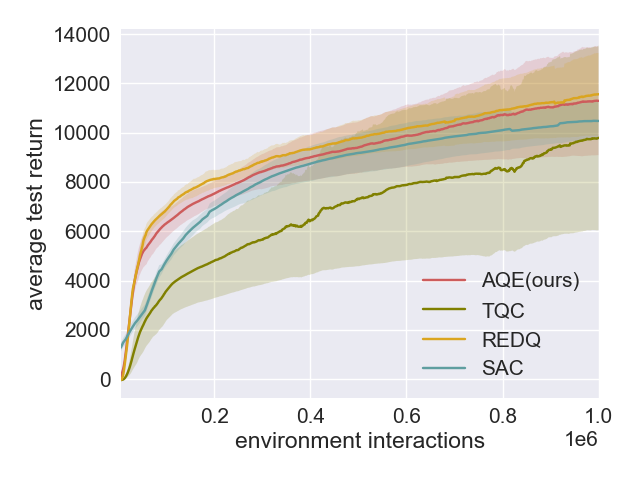}
	\caption{HalfCheetah-v2}
	\label{fig:same-hp-halfcheetah}
\end{subfigure}
\begin{subfigure}{0.329\textwidth}
	\centering
	\includegraphics[width=0.99\linewidth]{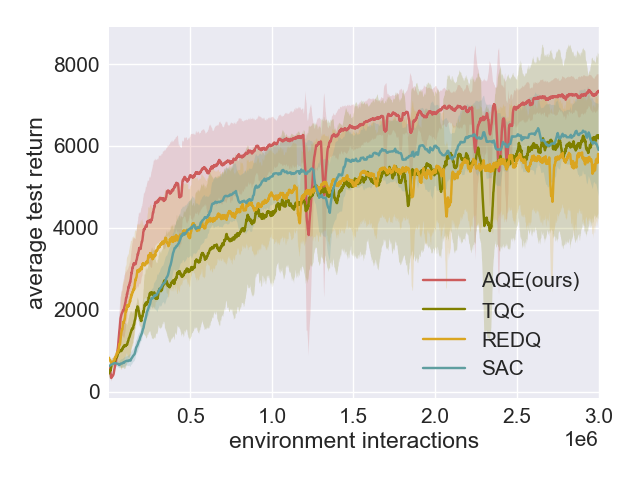}
	\caption{Ant-v2}
	\label{fig:same-hp-ant}
\end{subfigure}
\begin{subfigure}{0.329\textwidth}
	\centering
	\includegraphics[width=0.99\linewidth]{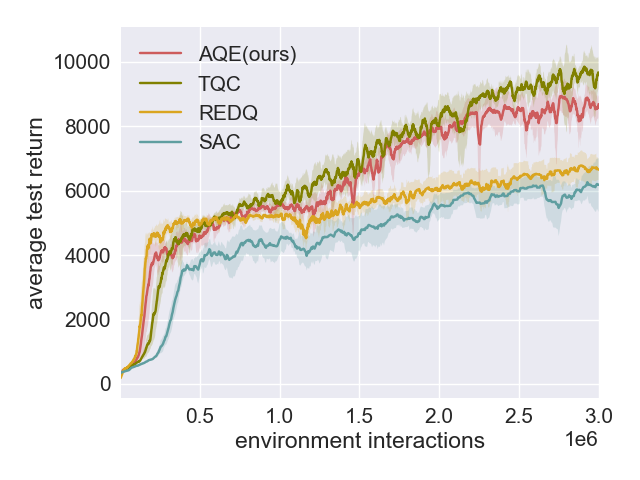}
	\caption{Humanoid-v2}
	\label{fig:same-hp-humanoid}
\end{subfigure}
\caption{Performance for AQE and TQC using same hyper-parameters across the five environments. AQE uses $K=16$ and TQC uses atoms = 2 per critic.}
\label{fig:same-hp}
\end{figure}

\begin{table}[ht]
    \centering
    \caption{Early-stage performance comparison of SAC, TQC, REDQ and AQE when AQE and TQC using the same hyperparameters across the environments. 
    On average, AQE performs 2.71 times better than SAC, 1.59 times better than TQC and 1.02 times better than REDQ.}
    \vskip 0.1in
    \begin{tabular}{l c c c c c c c}
         Amount of data & SAC & TQC & REDQ & AQE & AQE/SAC & AQE/TQC & AQE/REDQ \\
         \hline
         Hopper at 100K & 1456 & 1719 & 2747 & 2294 & 1.58 & 1.33 & 0.84 \\
         Walker2d at 100K & 501 & 1215 & 1810 & 2150 & 4.29 & 1.77 & 1.19 \\
         HalfCheetah at 100K & 3055 & 3594 & 6876 & 6325 & 2.10 & 1.76 & 0.92 \\
         Ant at 250K & 2107 & 2344 & 3279 & 4153 & 1.97 & 1.77 & 1.27 \\
         Humanoid at 250K & 1094 & 3038 & 4535 & 3973 & 3.63 & 1.31 & 0.88 \\
         \hline
         Average at early stage & - & - & - & - & 2.71 & 1.59 & 1.02 \\
    \end{tabular}
    \label{tab:early-performance-same-hp}
\end{table}

\begin{table}[!ht]
    \centering
    \caption{Late-stage performance comparison of SAC, TQC, REDQ and AQE when AQE and TQC using the same hyperparameters across the environments. 
    On average, AQE performs 16\% better than SAC, 9\% better than TQC and 11\% times better than REDQ.}
    \vskip 0.1in
    \begin{tabular}{l c c c c c c c}
         Amount of data & SAC & TQC & REDQ & AQE & AQE/SAC & AQE/TQC & AQE/REDQ \\
         \hline
         Hopper at 1M & 3282 & 2024 & 2954 & 2404 & 0.73 & 1.19 & 0.81 \\
         Walker2d at 1M & 4134 & 5532 & 4637 & 5517 & 1.33 & 1.00 & 1.19 \\
         HalfCheetah at 1M & 10475 & 9792 & 11562 & 11293 & 1.08 & 1.15 & 0.98 \\
         Ant at 3M & 5903 & 6186 & 5785 & 7345 & 1.24 & 1.19 & 1.27 \\
         Humanoid at 3M & 6177 & 9593 & 6649 & 8680 & 1.41 & 0.91 & 1.31 \\
         \hline
         Average at late stage & - & - & - & - & 1.16 & 1.09 & 1.11 \\
    \end{tabular}
    \label{tab:late-performance-same-hp}
\end{table}

\clearpage

\section{Additional Results for AQE, TQC, REDQ and SAC in DeepMind Control Suite Benchmark}

Figure \ref{fig:dmc-results} presents the performance of AQE, TQC, REDQ and SAC for 9 DeepMind Control Suite (DMC) environments.  We can see that AQE continues to outperform TQC except for Humanoid run environment,  where TQC performs better than AQE in the final stage training. AQE and REDQ have comparable results in some of the DMC environments, however, AQE usually outperforms REDQ in the more challenging environments, such as Hopper-hop, Humanoid-run and Quadruped-run. We report detailed early-stage and late-stage performance comparisons of all algorithms in Table \ref{tab:early-performance-same-hp-dmc} and Table \ref{tab:late-performance-same-hp-dmc}. On average, in the early stage of training, AQE performs 13.71 times better than SAC, 7.59 times better than TQC and 1.02 times better than REDQ. In the late-stage training, on average, AQE performs 1.37 times better than SAC, 1.08 times better than TQC and 1.03 times better than REDQ.

\begin{table}[ht]
    \centering
    \caption{Early-stage performance comparison of SAC, TQC, REDQ and AQE when AQE and TQC using the same hyperparameters across the DMC environments. 
    On average, AQE performs 13.71 times better than SAC, 7.59 times better than TQC and 1.02 times better than REDQ.}
    \vskip 0.1in
    \begin{tabular}{l c c c c c c c}
         Amount of data & SAC & TQC & REDQ & AQE & AQE/SAC & AQE/TQC & AQE/REDQ \\
         \hline
         Cheetah-run at 100K & 205 & 235 & 317 & 339 & 1.65 & 1.44 & 1.07 \\
         Fish-swim at 100K & 121 & 149 & 234 & 230 & 1.90 & 1.54 & 0.98 \\
         Hopper-hop at 100K & 2 & 11 & 50 & 64 & 32 & 5.81 & 1.28 \\
         Quadruped-walk at 100K & 116 & 172 & 452 & 341 & 2.94 & 1.98 & 0.75 \\
         Quadruped-run at 100K & 114 & 111 & 294 & 284 & 2.49 & 2.56 & 0.97 \\
         Walker-run at 100K & 305 & 372 & 468 & 457 & 1.50 & 1.23 & 0.98 \\
         Humanoid-stand at 100K & 5 & 5 & 37 & 52 & 10.4 &10.4& 1.41 \\
         Humanoid-walk at 100K & 1 & 1 & 57 & 40 & 40 & 40 & 0.70 \\
         Humanoid-run at 250K & 2 & 18 & 59 & 61 & 30.5 & 3.39 & 1.03 \\
         \hline
         Average at early stage & - & - & - & - & 13.71 & 7.59 & 1.02 \\
    \end{tabular}
    \label{tab:early-performance-same-hp-dmc}
\end{table}

\begin{table}[ht]
    \centering
    \caption{Late-stage performance comparison of SAC, TQC, REDQ and AQE when AQE and TQC using the same hyperparameters across the DMC environments. 
    On average, AQE performs 1.37 times better than SAC, 1.08 times better than TQC and 1.03 times better than REDQ.}
    \vskip 0.1in
    \begin{tabular}{l c c c c c c c}
         Amount of data & SAC & TQC & REDQ & AQE & AQE/SAC & AQE/TQC & AQE/REDQ \\
         \hline
         Cheetah-run at 1M & 734 & 829 & 844 & 856 & 1.17 & 1.03 & 1.01 \\
         Fish-swim at 1M & 639 & 722 & 753 & 747 & 1.17 & 1.03 & 0.99 \\
         Hopper-hop at 1M & 293 & 256 & 279 & 294 & 1.00 & 1.15 & 1.05 \\
         Quadruped-walk at 1M & 871 & 948 & 949 & 948 & 1.09 & 1.00 & 1.00 \\
         Quadruped-run at 1M & 676 & 893 & 904 & 928 & 1.37 & 1.04 & 1.03 \\
         Walker-run at 1M & 660 & 780 & 826 & 808 & 1.22 & 1.04 & 0.98 \\
         Humanoid-stand at 1M & 323 & 429 & 547 & 546 & 1.69 & 1.27 & 1.00 \\
         Humanoid-walk at 1M & 325 & 427 & 596 & 576 & 1.77 & 1.35 & 0.97 \\
         Humanoid-run at 4.5M & 146 & 324 & 216 & 271 & 1.86 & 0.84 & 1.25 \\
         \hline
         Average at late stage & - & - & - & - & 1.37 & 1.08 & 1.03 \\
    \end{tabular}
    \label{tab:late-performance-same-hp-dmc}
\end{table}

\begin{table}[htb]
    \begin{center}
    \caption{Sample efficiency comparison of SAC, TQC, REDQ and AQE. The numbers show the amount of data collected when the specified performance level is reached (roughly corresponding to 90\% of SAC's final performance). The last three columns show how many times AQE is more sample efficient than SAC, TQC and REDQ in reaching that performance level.} 
    \vskip 0.1in
    \begin{tabular}{l c c c c c c c}
    Performance & SAC & TQC & REDQ & AQE & AQE/SAC & AQE/TQC & AQE/REDQ \\
    \hline
    Cheetah-run at 700 & 746K & 440K & 506K & 350K & 2.13 & 1.26 & 1.45 \\
    Fish-swim at 600 & 794K & 494K & 317K & 417K & 1.90 & 1.18 & 0.76 \\
    Hopper-hop at 250 & 580K & 856K & 451K & 371K & 1.56 & 2.31 & 1.22 \\
    Quadruped-walk at 800 & 844K & 301K & 302K & 236K & 3.58 & 1.28 & 1.28 \\
    Quadruped-run at 650 & 942K & 521K & 267K & 248K & 3.80 & 2.10 & 1.08 \\
    Walker-run at 600 & 516K & 201K & 156K & 174K & 2.97 & 1.16 & 0.90 \\
    Humanoid-stand at 250 & 626K & 429K & 279K & 342K & 1.83 & 1.25 & 0.82 \\
    Humanoid-walk at 300 & 820K & 523K & 279K & 300K & 2.73 & 1.74 & 0.93 \\
    Humanoid-run at 120 & 3940K & 1100K & 602K & 603K & 6.53 & 1.82 & 1.00 \\
    \hline
    Average & - & - & - & - & 3.00 & 1.57 & 1.05 \\
    \end{tabular}
    \label{tab:sample-efficiency-dmc}
    \end{center}
\end{table}

\clearpage

\section{Additional Results for AQE, SAC-5 and TQC-5}
Figure \ref{fig:ablation-utd-appendix} presents the performance of AQE, SAC-5 and TQC-5 for all the environments. SAC-5 and TQC-5 uses UTD ratio G = 5 for SAC and TQC, respectively. We can see that AQE continues to outperform both algorithms except for Humanoid, where TQC performs somewhat better than AQE in the final stage training. SAC becomes more sample efficient with $G=5$; however, AQE still beats SAC-5 by a large margin.

\begin{figure}[h!tb]
\centering
\begin{subfigure}{0.3\textwidth}
	\centering
	\includegraphics[width=0.99\linewidth]{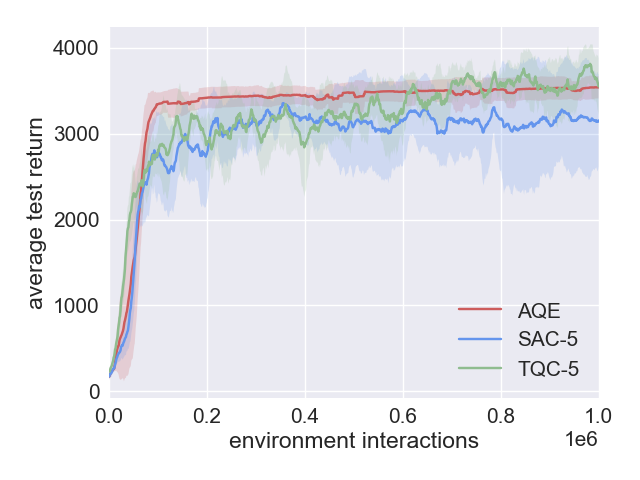}
	\caption{Performance, Hopper}
	\label{fig:hopper-abl-G}
\end{subfigure}
\begin{subfigure}{0.3\textwidth}
	\centering
	\includegraphics[width=0.99\linewidth]{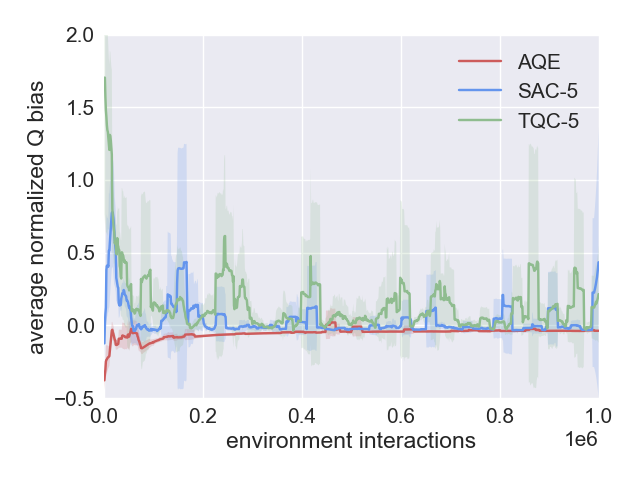}
	\caption{Average normalized bias}
    \label{fig:hopper-abl-G-qbias}
\end{subfigure}
\begin{subfigure}{0.3\textwidth}
	\centering
	\includegraphics[width=0.99\linewidth]{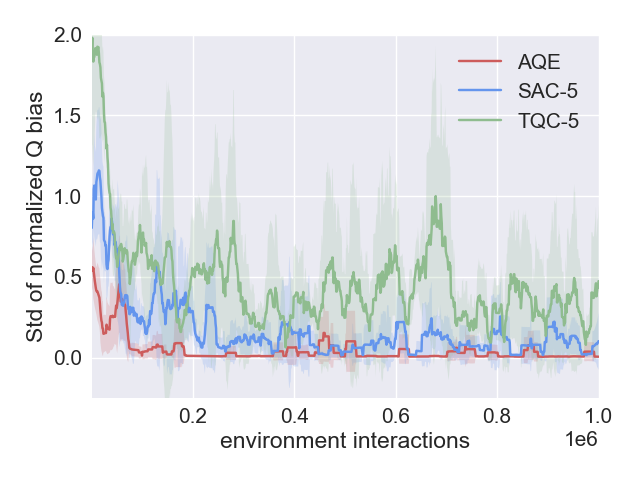}
	\caption{Std of normalized bias}
	\label{fig:hopper-abl-G-std}
\end{subfigure}

\begin{subfigure}{0.3\textwidth}
	\centering
	\includegraphics[width=0.99\linewidth]{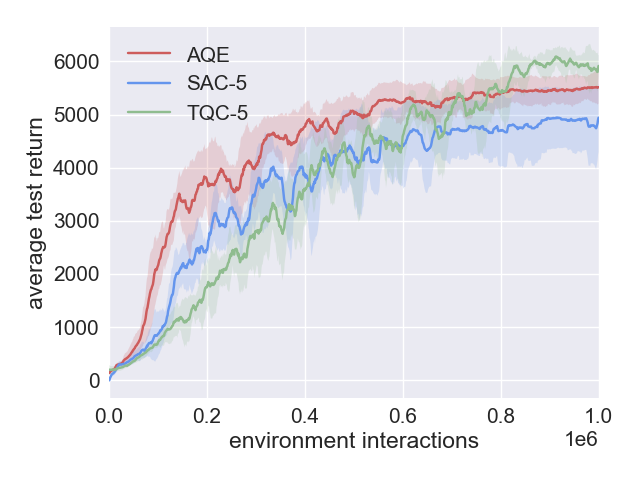}
	\caption{Performance, Walker}
	\label{fig:walker-abl-G}
\end{subfigure}
\begin{subfigure}{0.3\textwidth}
	\centering
	\includegraphics[width=0.99\linewidth]{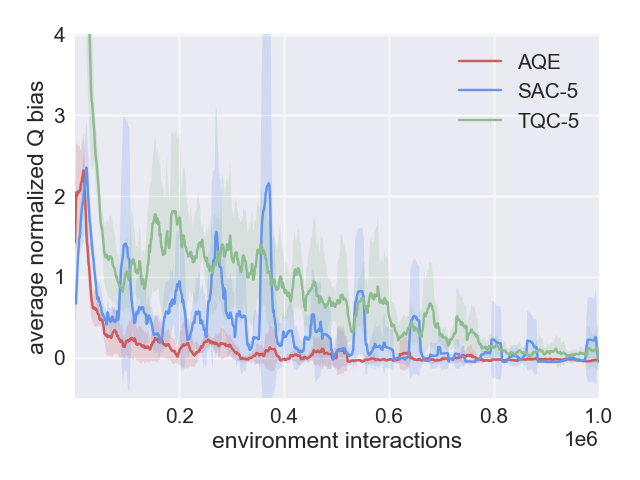}
	\caption{Average normalized bias}
    \label{fig:walker-abl-G-qbias}
\end{subfigure}
\begin{subfigure}{0.3\textwidth}
	\centering
	\includegraphics[width=0.99\linewidth]{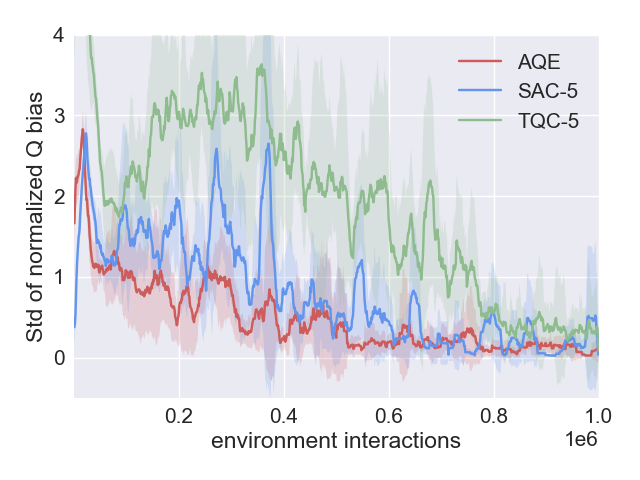}
	\caption{Std of normalized bias}
	\label{fig:walker-abl-G-std}
\end{subfigure}

\begin{subfigure}{0.3\textwidth}
	\centering
	\includegraphics[width=0.99\linewidth]{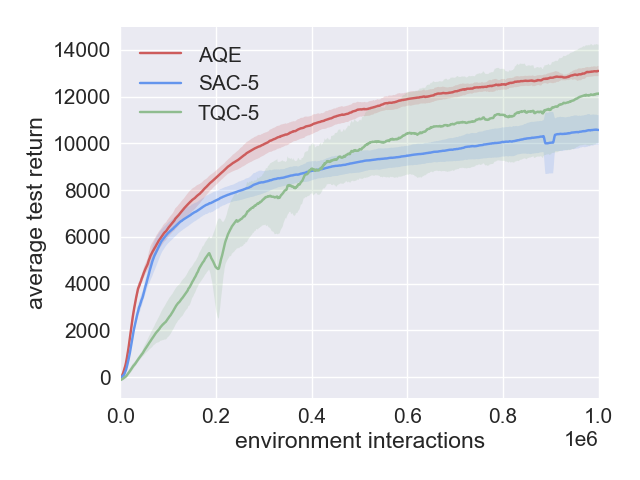}
	\caption{Performance, HalfCheetah}
	\label{fig:halfcheetah-abl-G}
\end{subfigure}
\begin{subfigure}{0.3\textwidth}
	\centering
	\includegraphics[width=0.99\linewidth]{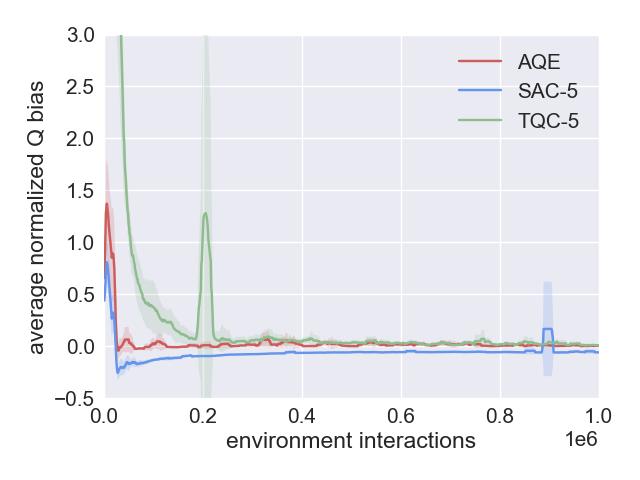}
	\caption{Average normalized bias}
    \label{fig:halfcheetah-abl-G-qbias}
\end{subfigure}
\begin{subfigure}{0.3\textwidth}
	\centering
	\includegraphics[width=0.99\linewidth]{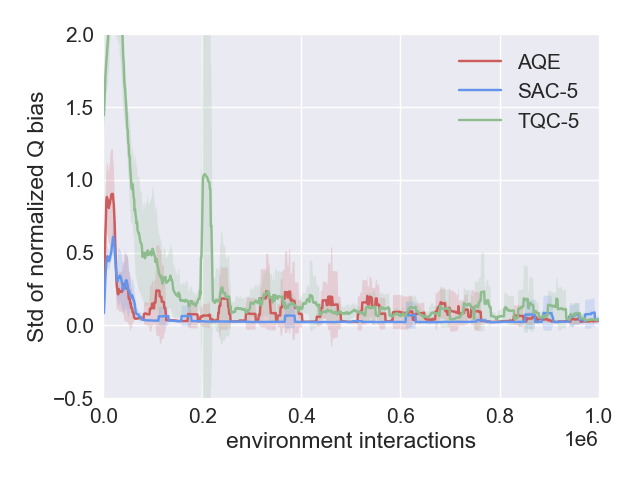}
	\caption{Std of normalized bias}
	\label{fig:halfcheetah-abl-G-std}
\end{subfigure}

\begin{subfigure}{0.3\textwidth}
	\centering
	\includegraphics[width=0.99\linewidth]{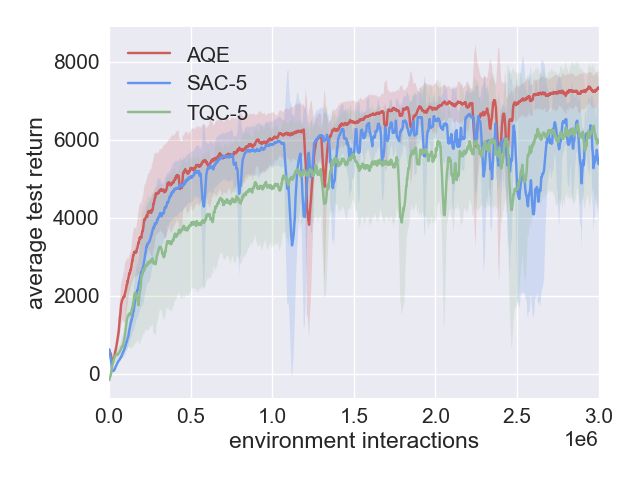}
	\caption{Performance, Ant}
	\label{fig:ant-abl-G}
\end{subfigure}
\begin{subfigure}{0.3\textwidth}
	\centering
	\includegraphics[width=0.99\linewidth]{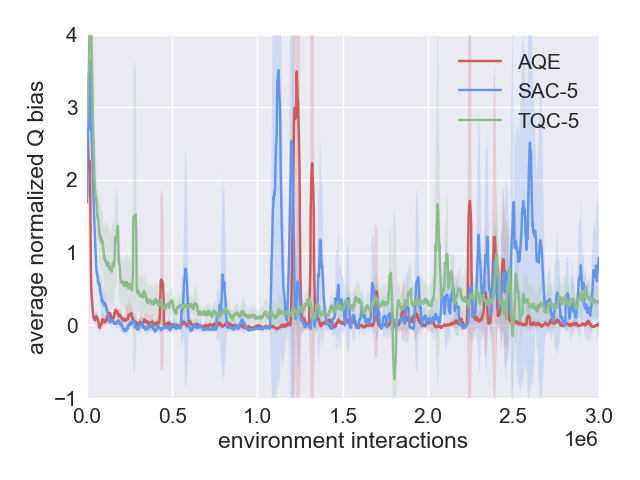}
	\caption{Average normalized bias}
    \label{fig:ant-abl-G-qbias}
\end{subfigure}
\begin{subfigure}{0.3\textwidth}
	\centering
	\includegraphics[width=0.99\linewidth]{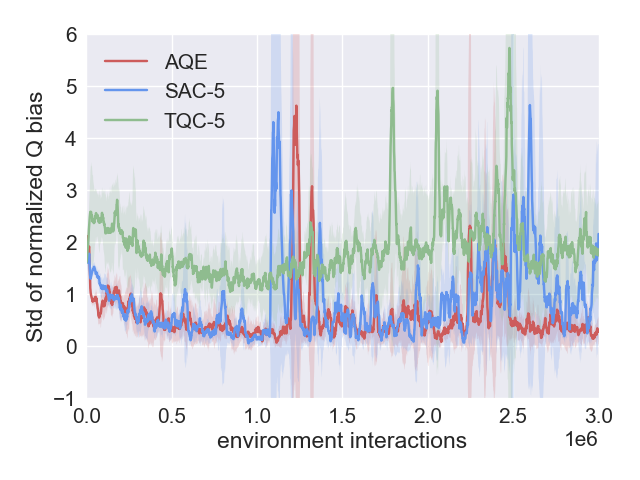}
	\caption{Std of normalized bias}
	\label{fig:ant-abl-G-std}
\end{subfigure}
\begin{subfigure}{0.3\textwidth}
	\centering
	\includegraphics[width=0.99\linewidth]{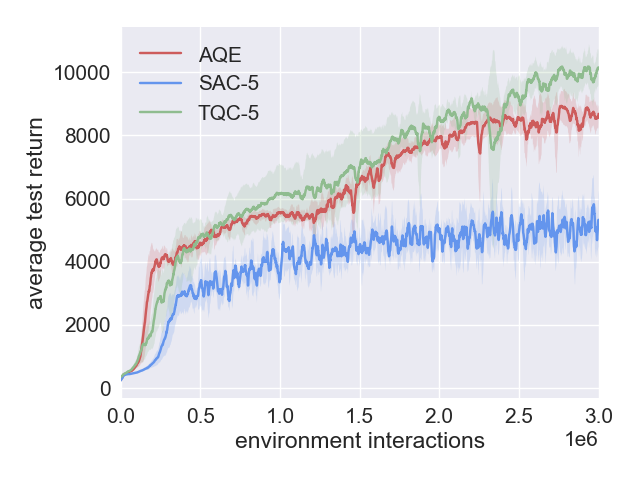}
	\caption{Performance, Humanoid}
	\label{fig:humanoid-abl-G}
\end{subfigure}
\begin{subfigure}{0.3\textwidth}
	\centering
	\includegraphics[width=0.99\linewidth]{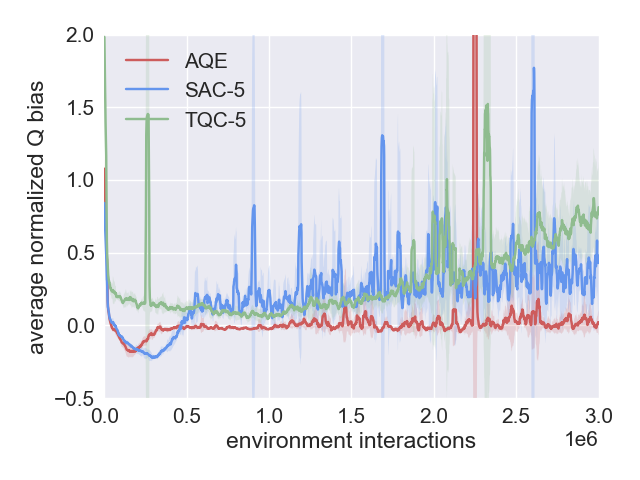}
	\caption{Average normalized bias}
    \label{fig:humanoid-abl-G-qbias}
\end{subfigure}
\begin{subfigure}{0.3\textwidth}
	\centering
	\includegraphics[width=0.99\linewidth]{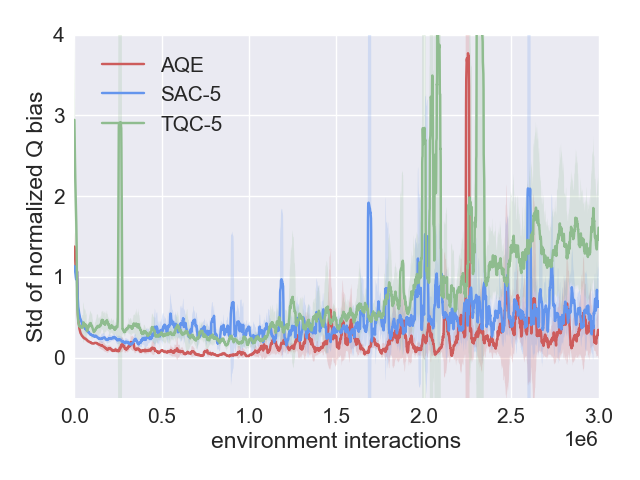}
	\caption{Std of normalized bias}
	\label{fig:humanoid-abl-G-std}
\end{subfigure}
\caption{Performance, average and std of normalized Q bias for AQE, SAC-5 and TQC-5. All of the algorithms in this experiment use UTD = 5.}
\label{fig:ablation-utd-appendix}
\end{figure}

\clearpage
\section{Additional Results for parameter $K$}
Due to lack of space, Figure \ref{fig:ablation} only compares  different AQE keep numbers $K$ for Ant. Figure \ref{fig:ablation-k-appendix} shows the performance, average estimation bias and standard deviation for all five environments.   Consistent with the theoretical result in Theorem 1, by decreasing $K$, we lower the average bias.

\begin{figure}[h!tb]
\centering
\begin{subfigure}{0.3\textwidth}
	\centering
	\includegraphics[width=0.99\linewidth]{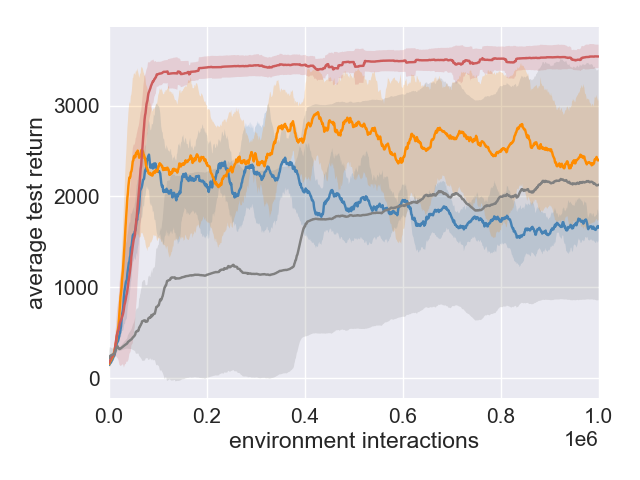}
	\caption{Performance, Hopper}
	\label{fig:hopper-abl-k}
\end{subfigure}
\begin{subfigure}{0.3\textwidth}
	\centering
	\includegraphics[width=0.99\linewidth]{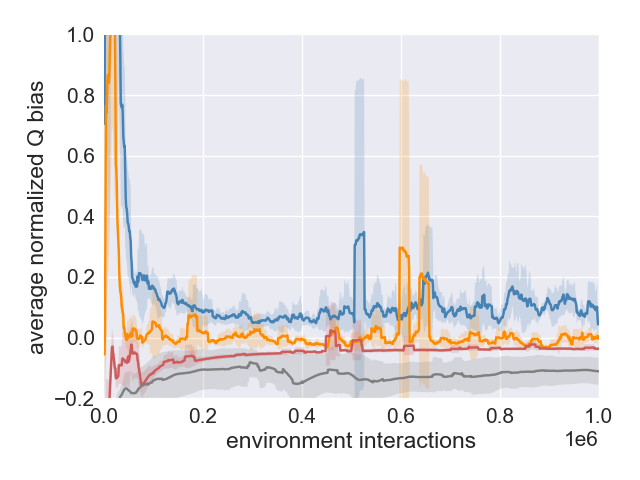}
	\caption{Average normalized bias}
    \label{fig:hopper-abl-k-qbias}
\end{subfigure}
\begin{subfigure}{0.3\textwidth}
	\centering
	\includegraphics[width=0.99\linewidth]{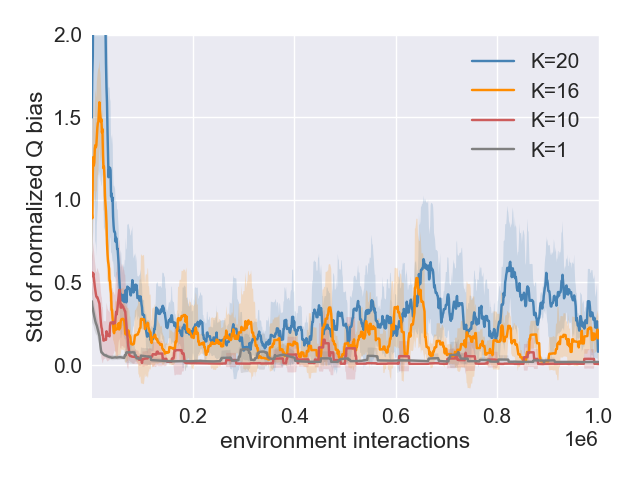}
	\caption{Std of normalized bias}
	\label{fig:hopper-abl-k-std}
\end{subfigure}

\begin{subfigure}{0.3\textwidth}
	\centering
	\includegraphics[width=0.99\linewidth]{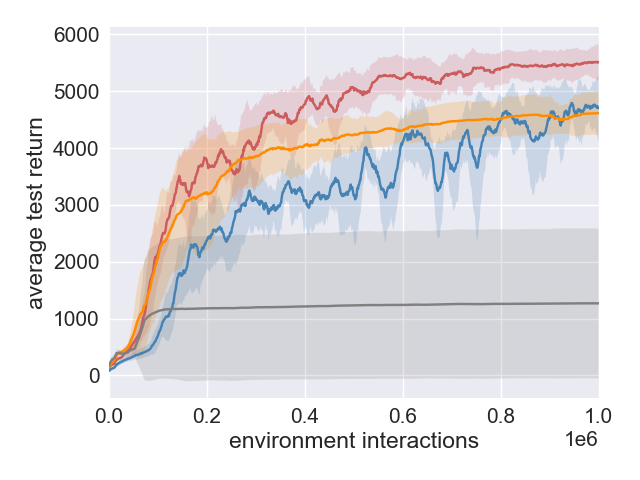}
	\caption{Performance, Walker}
	\label{fig:walker-abl-k}
\end{subfigure}
\begin{subfigure}{0.3\textwidth}
	\centering
	\includegraphics[width=0.99\linewidth]{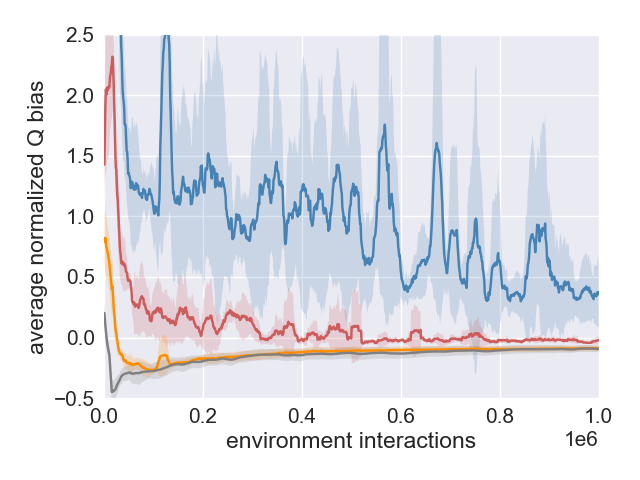}
	\caption{Average normalized bias}
    \label{fig:walker-abl-k-qbias}
\end{subfigure}
\begin{subfigure}{0.3\textwidth}
	\centering
	\includegraphics[width=0.99\linewidth]{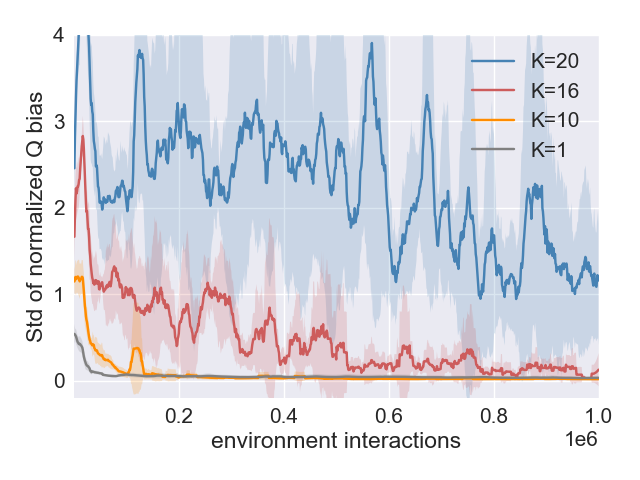}
	\caption{Std of normalized bias}
	\label{fig:walker-abl-k-std}
\end{subfigure}

\begin{subfigure}{0.3\textwidth}
	\centering
	\includegraphics[width=0.99\linewidth]{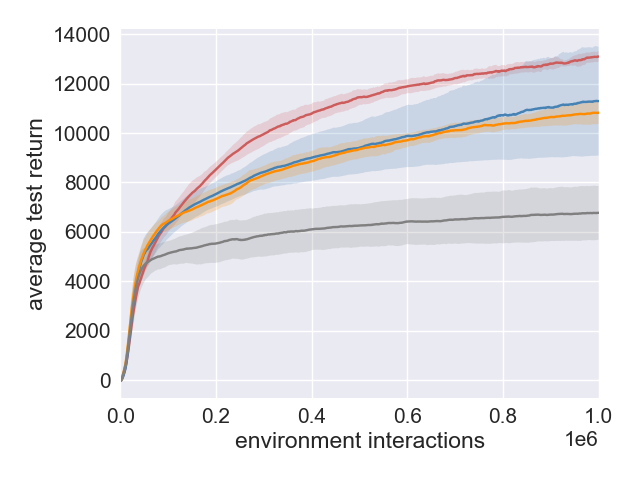}
	\caption{Performance, HalfCheetah}
	\label{fig:halfcheetah-abl-k}
\end{subfigure}
\begin{subfigure}{0.3\textwidth}
	\centering
	\includegraphics[width=0.99\linewidth]{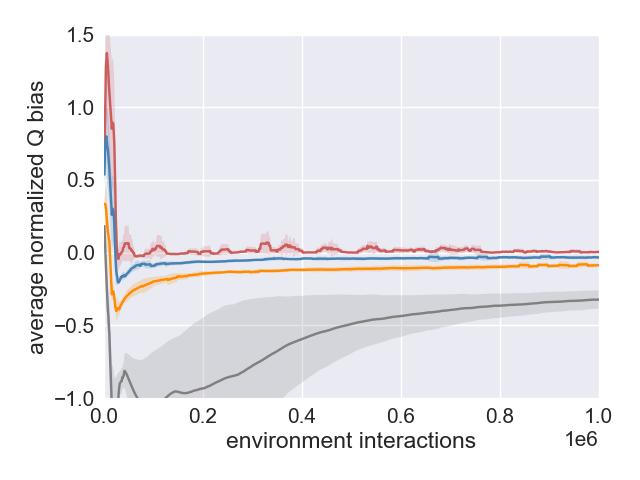}
	\caption{Average normalized bias}
    \label{fig:halfcheetah-abl-k-qbias}
\end{subfigure}
\begin{subfigure}{0.3\textwidth}
	\centering
	\includegraphics[width=0.99\linewidth]{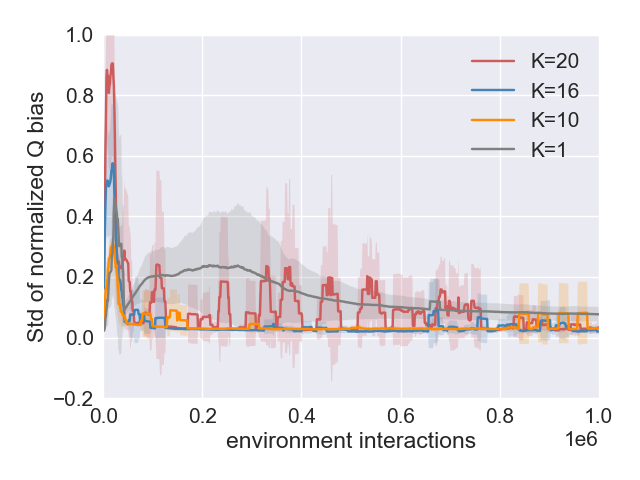}
	\caption{Std of normalized bias}
	\label{fig:halfcheetah-abl-k-std}
\end{subfigure}

\begin{subfigure}{0.3\textwidth}
	\centering
	\includegraphics[width=0.99\linewidth]{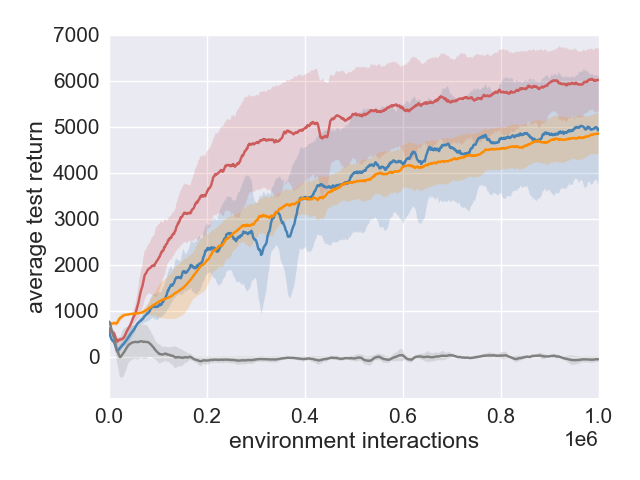}
	\caption{Performance, Ant}
	\label{fig:ant-abl-k}
\end{subfigure}
\begin{subfigure}{0.3\textwidth}
	\centering
	\includegraphics[width=0.99\linewidth]{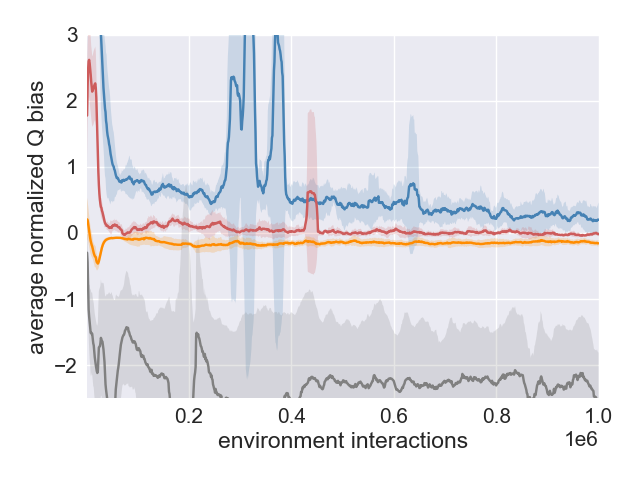}
	\caption{Average normalized bias}
    \label{fig:ant-abl-k-qbias}
\end{subfigure}
\begin{subfigure}{0.3\textwidth}
	\centering
	\includegraphics[width=0.99\linewidth]{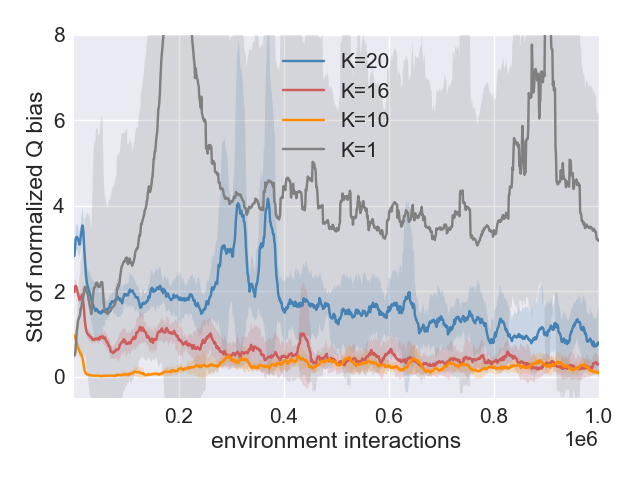}
	\caption{Std of normalized bias}
	\label{fig:ant-abl-k-std}
\end{subfigure}
\begin{subfigure}{0.3\textwidth}
	\centering
	\includegraphics[width=0.99\linewidth]{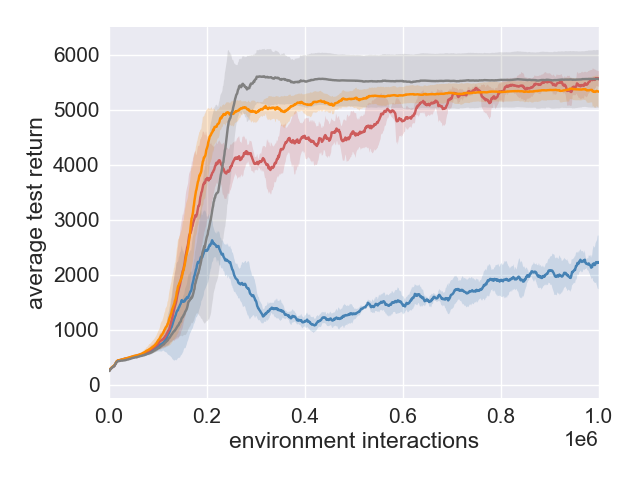}
	\caption{Performance, Humanoid}
	\label{fig:humanoid-abl-k}
\end{subfigure}
\begin{subfigure}{0.3\textwidth}
	\centering
	\includegraphics[width=0.99\linewidth]{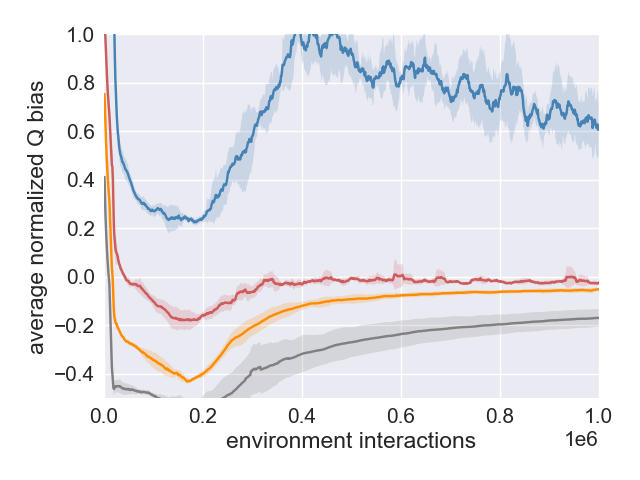}
	\caption{Average normalized bias}
    \label{fig:humanoid-abl-k-qbias}
\end{subfigure}
\begin{subfigure}{0.3\textwidth}
	\centering
	\includegraphics[width=0.99\linewidth]{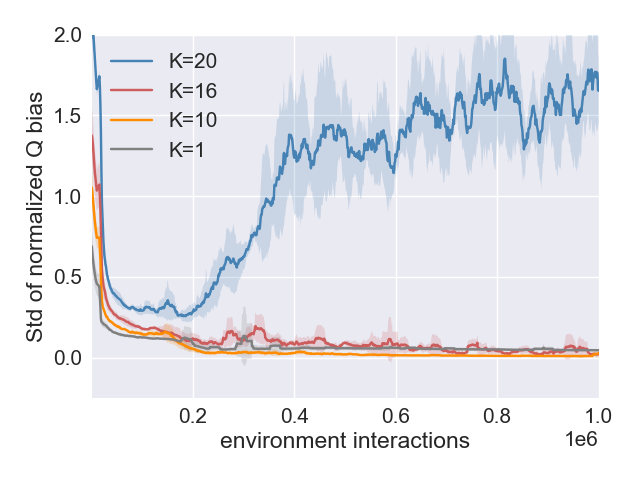}
	\caption{Std of normalized bias}
	\label{fig:humanoid-abl-k-std}
\end{subfigure}
\caption{Performance, average and std of normalized Q bias for AQE with different values of $K$.}
\label{fig:ablation-k-appendix}
\end{figure}

\clearpage

\section{Additional Results for Multi-head Architecture}
\label{ap:multihead}
Due to lack of space, Figure \ref{fig:ablation} only compares the different size of the ensemble $N$ and the number of heads $h$ for Ant. Figure \ref{fig:full-Nhead} shows the results for all five environments. We can see that the combination of $N=10, h=2$ and $N=20, h=1$ have comparable performance. However, $N=10$ and $h=2$ is faster in terms of computation time. 

\begin{figure}[h!tb]
\centering
\begin{subfigure}{0.329\textwidth}
	\centering
	\includegraphics[width=0.99\linewidth]{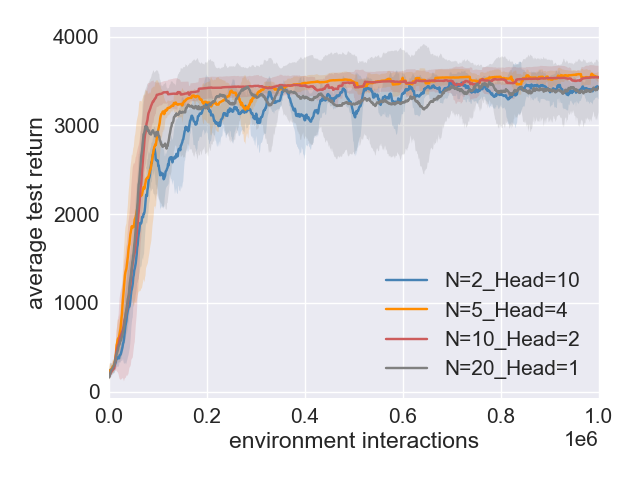}
	\caption{Hopper-v2}
	\label{fig:Nhead-hopper}
\end{subfigure}
\begin{subfigure}{0.329\textwidth}
	\centering
	\includegraphics[width=0.99\linewidth]{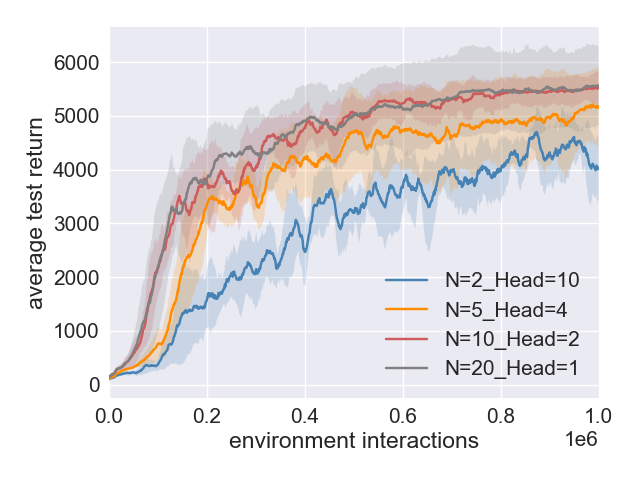}
	\caption{Walker2d-v2}
    \label{fig:Nhead-walker2d}
\end{subfigure}
\begin{subfigure}{0.329\textwidth}
	\centering
	\includegraphics[width=0.99\linewidth]{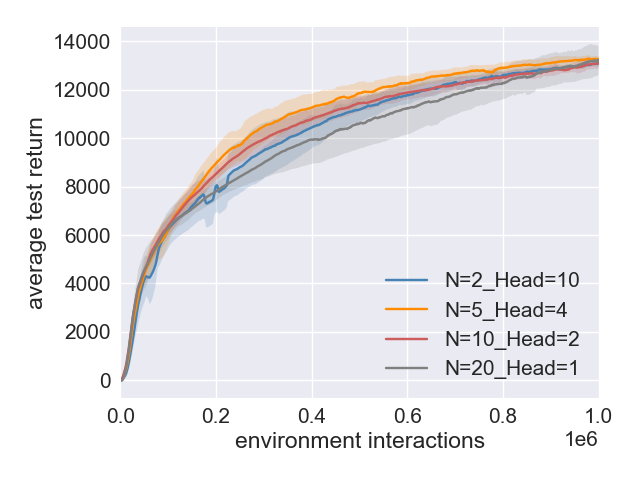}
	\caption{HalfCheetah-v2}
	\label{fig:Nhead-halfcheetah}
\end{subfigure}
\begin{subfigure}{0.329\textwidth}
	\centering
	\includegraphics[width=0.99\linewidth]{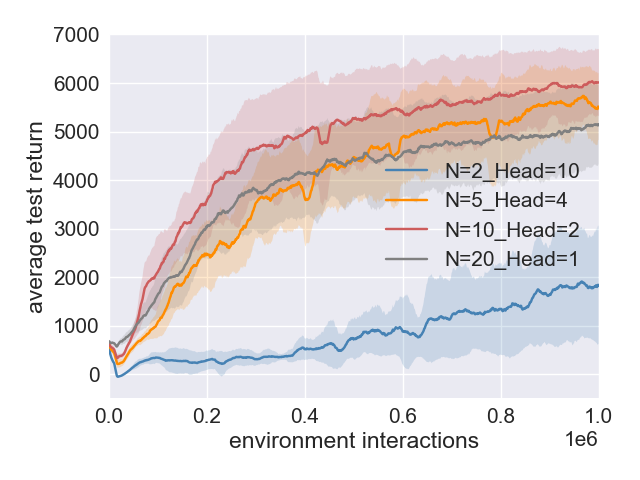}
	\caption{Ant-v2}
	\label{fig:Nhead-ant}
\end{subfigure}
\begin{subfigure}{0.329\textwidth}
	\centering
	\includegraphics[width=0.99\linewidth]{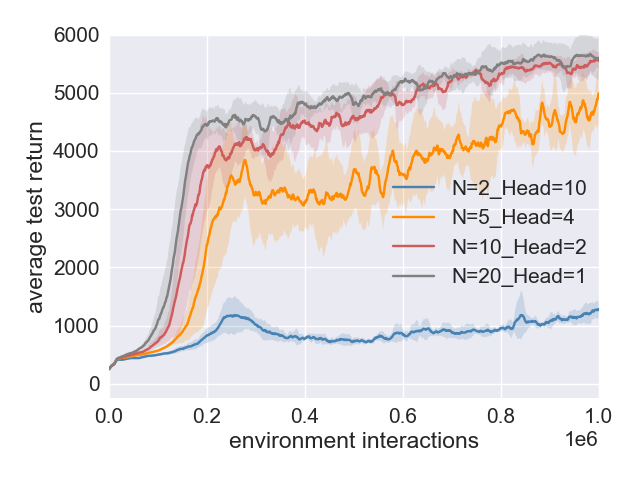}
	\caption{Humanoid-v2}
	\label{fig:Nhead-humanoid}
\end{subfigure}
\caption{Performance for AQE with different combinations of number of Q networks and number of heads.}
\label{fig:full-Nhead}
\end{figure}

Will the performance of REDQ match that of AQE if we also provide REDQ a multi-head architecture? Figure \ref{fig:redq-head} examines the performance of REDQ when it is endowed with the same multi-head architecture as AQE. We see that the performance of REDQ does not substantially improve. 

\begin{figure}[h!tb]
\centering
\begin{subfigure}{0.329\textwidth}
	\centering
	\includegraphics[width=0.99\linewidth]{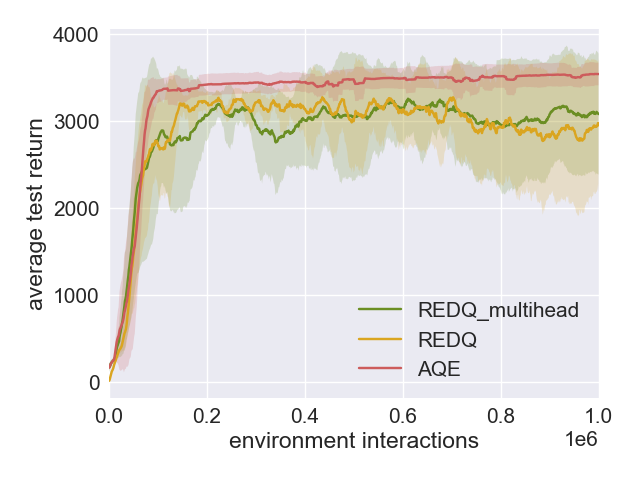}
	\caption{Hopper-v2}
	\label{fig:redq-head-hopper}
\end{subfigure}
\begin{subfigure}{0.329\textwidth}
	\centering
	\includegraphics[width=0.99\linewidth]{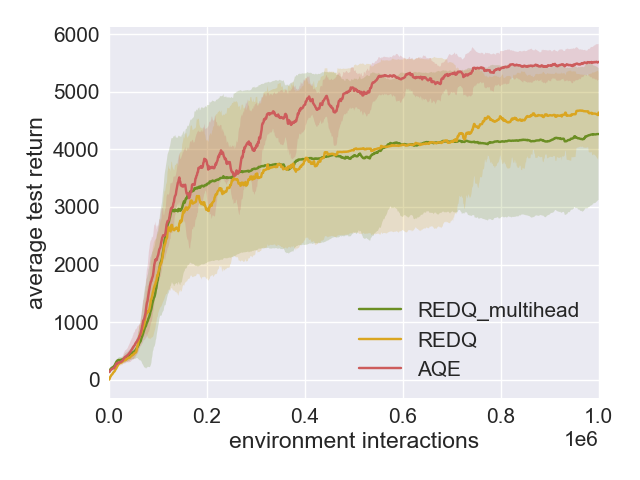}
	\caption{Walker2d-v2}
    \label{fig:redq-head-walker2d}
\end{subfigure}
\begin{subfigure}{0.329\textwidth}
	\centering
	\includegraphics[width=0.99\linewidth]{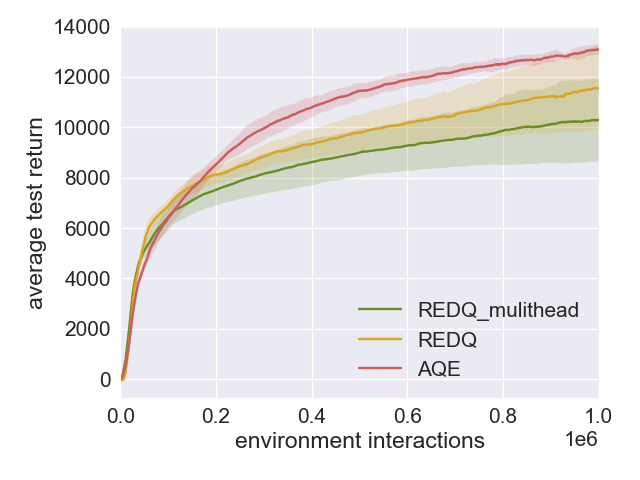}
	\caption{HalfCheetah-v2}
	\label{fig:redq-head-halfcheetah}
\end{subfigure}
\begin{subfigure}{0.329\textwidth}
	\centering
	\includegraphics[width=0.99\linewidth]{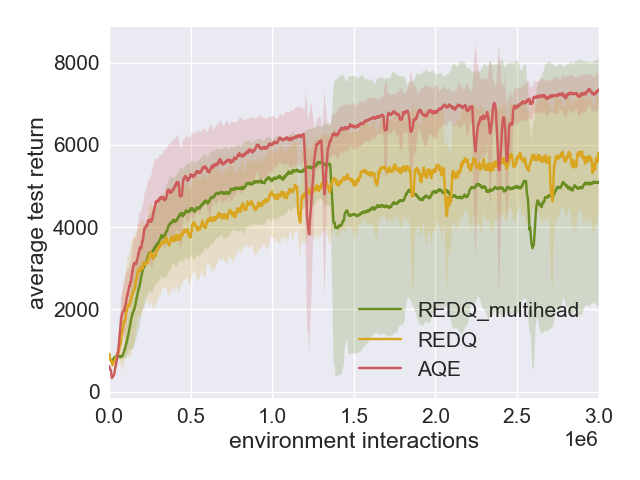}
	\caption{Ant-v2}
	\label{fig:redq-head-ant}
\end{subfigure}
\begin{subfigure}{0.329\textwidth}
	\centering
	\includegraphics[width=0.99\linewidth]{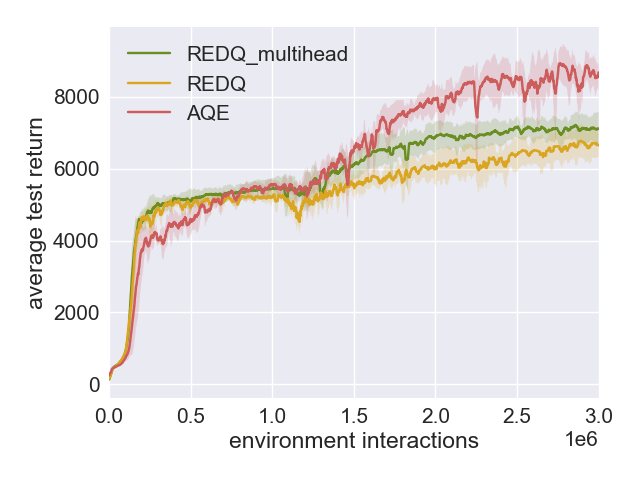}
	\caption{Humanoid-v2}
	\label{fig:redq-head-humanoid}
\end{subfigure}
\caption{Performance of REDQ with N=10 and heads = 2 as compared with REDQ and AQE.}
\label{fig:redq-head}
\end{figure}

\clearpage

\section{Theoretical Results}
In this section, we characterize how changing the size of the ensemble $N$ and the keep parameter $K$ affects the estimation bias term in the AQE algorithm. We will restrict our analysis to the tabular version of AQE shown in algorithm \ref{alg:tabular_aqe}.

Our analysis will follow similar lines of reasoning as \citet{maxmin-LanPFW20} and \citet{redq-chen2021randomized} which extends upon the theoretical framework introduced in \citet{thrun1993issues}.

For each $a\in\cA$, let $E_{K,N}(s,a)$ be the ensemble members in $\{1,\ldots,N\}$ with the $K$ lowest values of $Q^{j}(s, a)$, $j=1,\ldots,N$.
In the tabular case, the target for the Q networks take the form:
\begin{equation}\label{eq:q_targ}
    r + \gamma\max_{a'}\left(\frac{1}{K}\sum_{j \in E_{K,N}(s',a')} Q^{j}(s',a')\right).
\end{equation}

Define the \emph{post-update estimation bias} as
\begin{equation}
\begin{aligned}
    Z_{K, N} &:= r + \gamma\max_{a'}\left(\frac{1}{K}\sum_{j \in E_{K,N}(s',a')} Q^{j}(s',a')\right) - \left(r + \gamma\max_{a'}Q^{\pi}(s',a') \right) \\
    &= \gamma\left[\max_{a'} \left(\frac{1}{K}\sum_{j \in E_{K,N}(s',a')} Q^{j}(s',a')\right) -   \max_{a'}Q^{\pi}(s',a')\right]
\end{aligned}
\end{equation}
Under this definition, if $\E[Z_{K,N}] > 0$, then the expected post-update estimation bias is positive and there is a tendency for the positive bias to accumulate during updates. 
Similarly, if $\E[Z_{K,N}] < 0$, then the expected post-update estimation bias is negative and there is a tendency for the negative bias to accumulate during updates. Ideally, we would like 
$\E[Z_{K,N}] \approx 0$

Also let
\begin{equation}
    Q^j(s,a) = Q^{\pi}(s,a) + e^j(s,a)
\end{equation}
where $e^j(s,a)$ is an independent and identically distributed error term across all $j$'s and all $a$'s for each fixed $s$. We further assume that $\E[e^j(s,a)]=0$. Note that with this assumption
\begin{displaymath}
    \E \left[\frac{1}{N}\sum_{j=1}^N  Q^{j}(s,a)\right]- Q^{\pi}(s,a)  =0,
\end{displaymath}
that is the pre-update estimation bias is zero. The following theorem shows how the expected estimation bias changes with $N$ and $K$:

\begin{restatable}{theorem}{tabularthm}
\label{thm:tabular_thm}
The following results hold for $\E[Z_{K,N}]$:
\begin{enumerate}
    \item $\E[Z_{N,N}] \geq 0$ for all $N\geq 1$.
    \item $\E[Z_{K-1,N}] \leq \E[Z_{K,N}]$ for all $K \leq N$.
    \item $\E[Z_{K, N+1}] \leq \E[Z_{K, N}]$.
    \item Suppose that $e_{sa}^j \leq c$ for some $c > 0$ for all $s, a$ and $j$. Then there exists an $N$ sufficiently large and $K<N$ such that $\E[Z_{K,N}] < 0$.
\end{enumerate}
\end{restatable}
\begin{proofsketch}
Part 1 is a result of Jensen's Inequality, and Parts 2 and 3 can be shown by analyzing how the average of the $K$ smallest ensembles changes when adding an extra ensemble model. Given the first three parts, we only need to show that $\E[Z_{1, N}] < 0$ to show that there exists a $K$ for a sufficiently large $N$ where the expected bias is negative. See the next section for full proof.
\end{proofsketch}

Theorem 1 shows that we can control the expected post-update bias $\E[Z_{K,N}]$ through adjusting $K$ and $N$. More concretely, we can bring the bias term from above zero (i.e. over estimation) to under zero (i.e. under estimation) by decreasing $K$ and/or increasing $N$.

We note also that similar to \citet{redq-chen2021randomized}, we make very few assumptions on the error term $e_{s,a}$. This is in contrary to \citet{thrun1993issues} and \citet{maxmin-LanPFW20}, both of whom assume that the error term is uniformly distributed. 


\subsection{Tabular AQE with $N$ ensemble members and $d$ drops}
\label{append:tabular}

\begin{algorithm}[ht]
\caption{Tabular AQE}
\label{alg:tabular_aqe}
\begin{algorithmic}[1]
\INPUT $Q^j(s,a)$ for all $s\in\cS$, $a\in\cA$, $j=1,\dots,N$.
\Repeat
\State For some state $s\in\cS$, choose $a\in \mathcal{A}$ based on $\big\{Q^j(s,a)\big\}_{j=1}^N$, observe $r$, $s'$.
\State For each $a'\in\cA$, let $E_{K, N}(s',a')$ be the ensemble members in $\{1,\ldots,N\}$ with the $K$ lowest values of $Q^{j}(s', a')$, $j=1,\ldots,N$.
\State Get target
\[
y=r + \gamma\max_{a'\in \mathcal{A}}\frac{1}{K}\sum_{j \in E_{K, N}(s',a') } Q^{j}(s', a')
\]
\For{$j=1,\ldots,N$}
    \State Update each $Q^j(s,a)$
    \[
    Q^j(s, a) \gets Q^j(s, a) + \alpha(y - Q^j(s, a))
    \]
\EndFor
\State $s \gets s'$
\Until{end}
\end{algorithmic}
\end{algorithm}

\clearpage
\section{Proofs}
We first present the following lemma:

\begin{lemma}[\citealp{redq-chen2021randomized}]\label{lemma:min}
Let $X_1, X_2, \dots$ be an infinite sequence of $i.i.d.$ random variables with cdf $F(x)$ and let $\tau = \inf{x: F(x) > 0}$. Also let $Y_N = \min\{X_1, X_2, \dots, X_N\}$. Then $Y_1, Y_2,\dots$ converges to $\tau$ almost surely.
\end{lemma}
\begin{proof}
See Appendix A.2 of \citet{redq-chen2021randomized}
\end{proof}

\tabularthm*

\begin{proof}
\begin{enumerate}
    \item By definition,
    \begin{equation}
        \begin{aligned}
            \E[Z_{N,N}] &= \gamma \E\left[\max_{a'} \left(\frac{1}{N}\sum_{j=1}^{N} Q^j(s',a')\right) - \max_{a'}Q^{\pi}(s',a')\right] \\
            &\geq \gamma\left[\max_{a'}E\left[ \left(\frac{1}{N}\sum_{j=1}^{N} Q^j(s',a')\right) \right]- \max_{a'}Q^{\pi}(s',a')\right] \\
            &= \gamma\left[\max_{a'}Q^{\pi}(s',a') - \max_{a'}Q^{\pi}(s',a')\right] = 0
        \end{aligned}
    \end{equation}
\item Let
\begin{equation}
    \Bar{Q}_{K,N}(s,a) = \frac{1}{K}\sum_{j\in E_{K,N}}Q^j(s,a).
\end{equation}
Since for any state $s$, $\max_a \Bar{Q}_{K+1,N}(s,a)\geq \max_a \Bar{Q}_{K,N}(s,a)$,
\begin{equation}
    \begin{aligned}
       \E[Z_{K+1,N}] &= \gamma\E\left[\max_{a'} \Bar{Q}_{K+1,N}(s',a') - \max_{a'}Q^{\pi}(s',a')\right] \\
       &\geq \gamma\E\left[\max_{a'} \Bar{Q}_{K,N}(s',a') - \max_{a'}Q^{\pi}(s',a')\right] \\
       &= \E[Z_{K,N}]
    \end{aligned}
\end{equation}
\item 
Comparing $\E[Z_{K,N}]$ and $\E[Z_{K,N+1}]$ is equivalent to comparing $\Bar{Q}_{K,N}(s,a)$ and $\Bar{Q}_{K,N+1}(s,a)$. Since $e^j(s,a)$ is i.i.d., by extension $Q^j(s,a)$ is also i.i.d. for $j=1,2,\cdots$. Suppose $Q^j(s,a)$ is drawn from some probability distribution $F$, then given $\Bar{Q}_{K,N}(s,a)$, $\Bar{Q}_{K,N+1}(s,a)$ can be calculated by generating an additional $Q^i(s,a)$ from $F$. The new sample $Q^i(s,a)$ affects the calculation of $\Bar{Q}_{K,N+1}(s,a)$ under the following two cases:
\begin{itemize}
    \item If $Q^i(s,a) > \max_{j\in E_{K,N}}Q^j(s,a)$, then the lowest $K$ values remain unchanged hence $\Bar{Q}_{K,N}(s,a)=\Bar{Q}_{K,N+1}(s,a)$.
    \item Else if $Q^i(s,a) \leq \max_{j\in E_{K,N}}Q^j(s,a)$, then $\max_{j\in E_{K,N}}Q^j(s,a)$ would be removed from and $Q^i(s,a)$ would be added to the set of lowest $K$ values, therefore $\Bar{Q}_{K,N+1}(s,a)\leq \Bar{Q}_{K,N}(s,a)$.
\end{itemize}
Combining the two cases $\Bar{Q}_{K,N+1}(s,a)\leq \Bar{Q}_{K,N}(s,a)$, therefore $\E[Z_{K,N+1}]\leq \E[Z_{K,N}]$

\item Since $\E[Z_{N, N}] \geq 0$, $\E[Z_{K, N}] \leq \E[Z_{K+1, N}]$ and $\E[Z_{K, N+1}] \leq \E[Z_{K, N}]$. It is suffice to show that $\E[Z_{1, N}] < 0$ for some $N$. The rest of the proof largely follows Theorem 1 of \citet{redq-chen2021randomized}. 

Let $\tau=\inf\{x: F_a(x) > 0\}$ where $F_a(x)$ is the cdf of $Q^j(s,a)$, $j=1,2,\dots$. By Lemma 1, $\Bar{Q}_{1,N}(s,a) = \min_{1\leq j\leq N} Q^j(s,a)$ converges almost surely to to $\tau_a$ for each $a$. Since the action space is finite, it then follows that $\max_a \Bar{Q}_{1,N}(s,a)$ converges almost surely to to $\tau=\max_a \tau_a$. Due to our assumption that $e^j(s,a)\leq c$ and that $Q^{\pi}(s,a)$ is finite, it then follows that $\max_a \Bar{Q}_{1,N}(s,a)$ is also bounded above. By Part 3 of the theorem, $\Bar{Q}_{1,N}(s,a)$ is monotonoically decreasing w.r.t. $N$. and since $\max_a \Bar{Q}_{1,N}(s,a)$ is also bounded above and converges almost surely to $\tau$, we have
\begin{equation}
    \begin{aligned}
       \E[Z_{1,N}] &= \gamma\left(\E[\max_a \min_{1\leq j\leq N}Q^j(s,a)] - \max_a Q^{\pi}(s,a) \right) \\
       &= \gamma \left(\E[\max_a Y_a^N] - \max_a Q^{\pi}(s,a)\right)\overset{N\to\infty}{\longrightarrow} \gamma\left(\max_a \tau_a - \max_a Q^{\pi}(s,a)\right) < 0
    \end{aligned}
\end{equation}
where the last equality follows from the assumption that the error $e^j(s,a)$ is non-trivial, and hence $\tau_a < \max_a Q^{\pi}(s,a)$ for all $a$. Therefore for a sufficiently large $N$, there exists a $1\leq K \leq N$ such that $\E_{K, N} < 0$.

\end{enumerate}

\end{proof}

\section{Computing Infrastructure}
Each experiment is run on a single Nvidia 2080-Ti GPU with CentOS Linux System.

\end{document}